\documentclass{article}
\usepackage[T1]{fontenc}
\usepackage[utf8]{inputenc}
\usepackage{lmodern}
\usepackage{fullpage}
\usepackage{url}            
\usepackage{booktabs}       
\usepackage{nicefrac}       
\usepackage{microtype}      
\usepackage{multirow}

\usepackage[T1]{fontenc}
\usepackage{amsmath,amssymb,amsfonts,amsthm,textcomp,graphicx,nicefrac,mathtools,mathrsfs, dsfont} 
\usepackage{tikz}

\usepackage{enumitem}
\usepackage{bbm}
\usepackage[utf8]{inputenc}
\usepackage{enumitem}

\usepackage{algorithm}
\usepackage[noend]{algpseudocode}

\makeatletter
\def\BState{\State\hskip-\ALG@thistlm}
\makeatother

\def\<{{\langle}}
\def\>{{\rangle}}

\renewcommand{\exp}[1]{\operatorname{exp}\left(#1\right)} 
\providecommand{\argmin}{\mathop\mathrm{arg\, min}}

\newcommand{\indi}{\mathds{1}}

\def\P{\mathbb{P}} 

\usepackage{graphicx}

\newtheoremstyle{dotless}{}{}{\itshape}{}{\bfseries}{}{ }{}
\theoremstyle{dotless}

\newtheorem*{defi}{Definition}
\theoremstyle{plain}

\newtheorem{myth}{Theorem}

\newtheorem{mylem}[myth]{Lemma}
\newtheorem{mycor}[myth]{Corollary}

\newtheoremstyle{named}{}{}{\itshape}{}{\bfseries}{.}{.5em}{#1 #3}
\theoremstyle{named}
\newtheorem*{namthm*}{Theorem}

\usepackage{chngcntr}

\usepackage[colorlinks=true, citecolor = blue]{hyperref}
\usepackage{cleveref}
\crefname{myth}{Theorem}{Theorems} 

\newcounter{parentnumber}

\usepackage[
backend=biber,
style=alphabetic,
maxalphanames=4,
useprefix=true,
sorting=anyt,
maxbibnames=99
]{biblatex}
\usepackage{enumitem}
\usepackage{graphicx}

\usepackage{todonotes}

\newcommand{\bp}{\overline{p}}

\usepackage{wrapfig}
\usepackage{csquotes}

\newcommand{\dk}{\boldsymbol{\bot}}
\newcommand{\hist}{\mathscr{H}}
\renewcommand{\phi}{\varphi}
\newcommand{\vs}{\mathcal{V}}
\newcommand{\pred}{\widehat{Y}}
\newcommand{\err}[1]{\widehat{Y}_{#1} \not\in\{\dk, Y_{#1}\} }
\renewcommand{\epsilon}{\varepsilon}

\title{Online Selective Classification with Limited Feedback}
\author{Aditya Gangrade, Anil Kag, Ashok Cutkosky, Venkatesh Saligrama \\ Boston University\\ {\small \texttt{gangrade@bu.edu}, \texttt{anilkag@bu.edu}, \texttt{ashok@cutkosky.com}, \texttt{srv@bu.edu} }}
\date{\vspace{-\baselineskip}}

\begin{document}

\maketitle

\begin{abstract}
    Motivated by applications to resource-limited and safety-critical domains, we study selective classification in the online learning model, wherein a predictor may abstain from classifying an instance. For example, this may model an adaptive decision to invoke more resources on this instance. Two salient aspects of the setting we consider are that the data may be non-realisable, due to which abstention may be a valid long-term action, and that feedback is only received when the learner abstains, which models the fact that reliable labels are only available when the resource intensive processing is invoked.
    
    Within this framework, we explore strategies that make few mistakes, while not abstaining too many times more than the best-in-hindsight error-free classifier from a given class. That is, the one that makes no mistakes, while abstaining the fewest number of times. We construct simple versioning-based schemes for any $\mu \in (0,1],$ that make most $T^\mu$ mistakes while incurring \smash{$\tilde{O}(T^{1-\mu})$} excess abstention against adaptive adversaries. We further show that this dependence on $T$ is tight, and provide illustrative experiments on realistic datasets. 
\end{abstract}

\section{Introduction}

Consider a low-power or battery-limited edge device, such as a sensor or a smart-speaker that receives a stream of classification requests. Due to the resource limitations, such a device cannot implement modern models that are needed for accurate decisions. Instead the device has access (e.g.~via an internet connection) to an accurate but resource-intensive model implemented on a cloud server, and may send queries to the cloud server in order to retain accuracy. Of course, this incurs costs such as latency and battery drain due to communication. The ideal operation of such a device should thus be to learn a rule that classifies `easy' instances locally, while sending harder ones to the cloud, thus maintaining accuracy whilst minimising the net resource consumption \cite{JMLR:v15:xu14a,nan2017adaptive}.

Selective classification \cite{Chow57, chow1970optimum} is a classical paradigm of relevance to such settings. The setup allows a predictor to abstain from classifying some instances (without incurring a mistake). This abstention models adaptive decisions to invoke more resource-intensive methods on subtle cases, like in the above example. The solution concept is relevant widely - for instance, it is relevant to adaptively recommending further (and costly) tests rather than offering a diagnosis in a medical scenario, or to recommending a human review instead of an alarm-or-not decision in security contexts. Two aspects of such settings are of particular interest to us. Firstly, the cheaper methods are typically not sufficient to realise the true labels, due to which abstention may be a long-term necessity. Secondly, a-priori reliable labels can only be obtained by invoking the resource intensive option, and thus feedback on whether a non-abstaining decision was correct is unavailable.

We propose online selective classification, with an emphasis on ensuring very few mistakes, to account for the need for very accurate decisions. Concretely, an adversary sequentially produces contexts and labels $(X_t, Y_t),$ and the learner uses the $X_t$s to produce a decision $\pred_t$ that may either be one of $K$ classes, or an abstention, which we represent as $\dk$. Feedback in the form of $Y_t$ is provided if and only if $\pred_t = \dk,$ and the learner incurs a mistake if $\pred_t$ was non-abstaining and did not equal $Y_t$. 

With the emphasis on controlling the total number of mistakes, we study regrets achievable when compared to the behaviour of the best-in-hindsight error-free selective classifier from a given class - that is, one that makes no mistakes, while abstaining the fewest number of times. Notice that our situation is non-realisable, and therefore this competitor may abstain in the long-run. The two metrics of importance here are the number of mistakes the learner makes, and its excess abstention over this competitor. An effective learner must control both abstention and mistakes, and it is not enough to make one small, e.g.~a learner that makes a lot of mistakes but incurs a very negative excess abstention is no good. This \emph{simultaneous} control of two regrets raises particular challenges.

We construct a simple scheme that, when competing against finite classes, simultaneously guarantees $O(T^\mu)$ mistakes and $O(T^{1-\mu})$ excess abstentions against adaptive adversaries (for any $\mu\in[0,1]$), and show that these rates are Pareto-tight \cite{osborne1994course}. We further show that against stochastic adversaries, the same rates can be attained with improved dependence of the regret bounds on the size of the class, and we also describe schemes that enjoy similar improvements against adaptive adversaries, but at the cost of the $T$-dependence of the regret bounds. The main schemes randomly abstain at a given rate in order to gain information, and otherwise play \smash{$\pred_t$} consistent with the `version space' of classifiers that have not been observed to make mistakes. For the adversarial case, the analysis of the scheme relies on a new `adversarial uniform law of large numbers'(ALLN) to argue that such methods cannot incur too many mistakes. This ALLN uses a self-normalised martingale concentration bound, and further yields an adaptive continuous approximation guarantee for the Bernoulli-sampling sketch in the sense of Ben-Eliezer \& Yogev \cite{ben2020adversarial, alon2021adversarial}. The theoretical exploration is complemented by illustrative experiments that implement our scheme on two benchmark datasets.

\subsection{Related Work}

Selective classification has been well studied in the batch setting, and many theoretical and methodological results have appeared \cite[e.g.][]{herbei-wegkamp, bartlett2008classification, el2010foundations, wiener2011agnostic,kalai2012reliable, cortes2016learning, Lei_class_w_confidence, geifman2019selectivenet, gangrade2021selective}. These batch results do not have strong implications for the online setting.

Cortes et al.~have studied selective classification in the online setting \cite{cortes2018online}, but with two differences from our setting. Firstly, rather than individually controlling mistakes and abstentions, the regret is defined according to the Chow loss, which adds up the number of mistakes and $c$ times the number of abstentions, where $c$ is a fixed cost parameter. Secondly (and more importantly) it is assumed that feedback is provided only when the learner \emph{does not} abstain, rather than only when it does. This difference arises from the underlying situations being modelled - Cortes et al.~view the abstention as a decision given to a user in which case no feedback is forthcoming, while we view it as a decision to invoke further processing. Both of the scenarios are reasonable, and so both of these explorations are valid, however it is unclear what implications one set of results have for the other. 

{A similar decision and feedback model as ours was proposed by Li et al. in the `knows what it knows' (KWIK) framework \cite{li2011knows}. The KWIK model, however, fundamentally views abstentions as a short term action, typically arguing that only a finite number of these are made. This is viable since Li et al. study this model in an essentially realisable setting, wherein the optimal labels are known to be essentially realised by a given class - notice that in such a case, a single abstention at an instance $x$ determines what value should be played there in the long run. Our interest however lies in the situation where this data cannot be represented in such a way, and such strategies are not viable since the labels may be noisy. Our work thus generalises the KWIK setting to non-realisable data, and to situations wherein abstention is a valid long-term action, as motivated in the introduction, by studying behaviour against competitors that may abstain.\footnote{The KWIK model also bears other significant differences. It posits an input parameter $\epsilon,$ and requires that the learner either abstains, or produces an $\epsilon$-accurate response. A notion of competitor is not invoked, and rather than studying regret, the number of abstentions needed to achieve this $\epsilon$-accuracy is studied.}}

While Szita and Szepesv\'{a}ri have extended the KWIK formulation to the agnostic case in a regression setting \cite{szita2011agnostic}, this work also focuses of limiting the number of abstentions to be finite rather than long-run abstentions. Concretely it is assumed that $Y_t = g(X_t) + \textrm{noise},$ for some function $g$, and the learner knows a class $\mathcal{H}$, and a bound $\Delta$ such some $h \in \mathcal{H}$ is $\Delta$-close to $g$ (in an appropriate norm). Using the knowledge of $\Delta,$ they describe schemes that have limited abstention, but at the cost of mistakes, by producing responses $\hat{Y}_t$ that are up to $(2+o(1))\Delta$ separated from $Y_t$. In contrast, in our formulation, contexts $X_t$ for which no function in $\mathcal{H}$ can represent the ground truth $g$ well would always be abstained upon. In addition to this work, trade-offs between mistakes and abstentions in a relaxed version of the KWIK framework have been considered \cite{zhang2016extended, sayedi2010trading, demaine2013learning}, and in particular the agnostic case has been explored by Zhang and Chaudhuri \cite{zhang2016extended}, but unlike our situation this relaxed KWIK model requires full-feedback to be available whether or not the learner abstains. Neu and Zhivotovskiy \cite{neu2020fast} also work in this relaxed model, and show that when comparing the standard loss of a \emph{non-abstaining} classifier against the Chow loss of an abstaining learner, regrets independent of time can be obtained.

Due to the limited feedback, our setting is related to partial-monitoring \cite[Ch. 37]{lattimore2020bandit}. Viewing actions as choices over functions, our setting has feedback graphs \cite{Mannor_feedback} that connect abstaining actions to every other action and themselves. The novelty with respect to partial-monitoring arises from the fact that we individually control two notions of losses, rather than a single one. It's unclear how to apply the generic partial-monitoring setup to this situation - indeed, na\"{i}vely, our game is only weakly observable in the sense of Alon et al.\cite{alon2015online}, and one would expect $\Omega(T^{2/3})$ regrets, while we can control both mistakes and excess abstention to $\tilde{O}(\sqrt{T})$. A limited feedback setting where two `losses' \emph{are} individually controlled is label-efficient prediction \cite{cesa2005minimizing}, where a learner must query in order to get feedback. However, in our setting, abstentions are both a way to gather feedback, and also necessary to prevent mistakes. That is, our competitor may abstain regularly, but makes few mistakes, while in this prior work the competitor does not abstain, but may make many mistakes. The resulting scenario is both qualitatively and quantitatively distinct, e.g.~in label-efficient prediction, the smallest symmetric rate of number of queries and excess mistakes is again $\Theta(T^{2/3})$.

\section{Setting, and Problem Formulation}\label{sec:setting}

\textbf{Setup} Let $\mathcal{X}$ be a feature space, $\mathcal{Y}$ a finite set of labels, and $\mathcal{F}$ a finite class of selective classifiers, which are $\mathcal{Y} \cup \{\dk\}$ valued. For simplicity, we assume that $\mathcal{F}$ contains the all abstaining classifier (i.e. the function $f_{\dk}$ such that $\forall x, f_{\dk}(x) = \dk$). We will denote $|\mathcal{F}| = N$. The setting may be described as a game between a learner and an adversary (or more prosaically, a data generating mechanism) proceeding in $T$ rounds. Also for simplicity, we will assume that $T$ is known to both the learner and the adversary in advance. The objects in this game are the context process, $X_t \in \mathcal{X}$, the label process $Y_t \in \mathcal{Y}$, the action process \smash{$\pred_t \in \mathcal{Y} \cup \{\dk\}$} and the feedback process $Z_t \in \mathcal{Y} \cup \{*\},$ where $* \not\in\mathcal{Y}$ is a trivial symbol. The information sets of the adversary and learner up to the $t$th round are respectively \smash{$\hist^{\mathfrak{A}}_{t-1} := \{(X_s, Y_s, \widehat{Y}_s): s<t\},$ and $\hist^{\mathfrak{L}}_{t-1} := \{(X_s, \widehat{Y}_s, Z_s): s < t\}.$}\\

\textbf{The Game} For each round $t \in [1:T],$ the adversary produces a context and a label $(X_t, Y_t)$ on the basis its history \smash{$\hist^{\mathfrak{A}}_{t-1}$}. The learner observes only the context, $X_t$, and on the basis of this and its history \smash{$\hist^{\mathfrak{L}}_{t-1}$, produces an action $\pred_t$. We will say that this action is an abstention if $\pred_t = \dk,$} and that it is a prediction otherwise. If the action was an abstention, set $Z_t = Y_t$, and otherwise to $*$. The learner then observes $Z_t,$ and the round concludes. Notice that since $Z_t$ is a deterministic function of $Y_t$ and \smash{$\widehat{Y}_t$, and since the adversary observes both, $\hist^{\mathfrak{L}}_{t-1}$ can be determinstically generated from $\hist^{\mathfrak{A}}_{t-1}$}. Due to the same reason, \smash{$\widehat{Y}_t$ and $Y_t$ are conditionally independent given $(X_t, \hist^{\mathfrak{A}}_{t-1})$.}\\

\textbf{Adversaries} are characterised by a sequence of conditional laws on $(X_t,Y_t)$ given $\hist^{\mathfrak{A}}_{t-1}$ (and $T,\mathcal{F}$). In the following we will explicitly consider two classes of such laws:\begin{itemize}[nosep]
    \item Stochastic Adversary: $(X_t, Y_t)$ are drawn according to a fixed law, $P,$ unknown to the learner, independently of $\hist^{\mathfrak{A}}_{t-1}$.
    \item Adaptive Adversary: $(X_t,Y_t)$ are arbitrary random variables with $\hist^{\mathfrak{A}}_{t-1}$-measurable laws.
\end{itemize}
We will denote a generic class of adversaries as $\mathscr{C}$.\\

\textbf{Performance Metrics} The two principal quantities of interest are the number of mistakes made by the learner, and the number of times it has abstained. We will denote these as \begin{align*}
    M_T := \sum_{t \le T} \indi\{ \widehat{Y}_t \not\in \{ \dk, Y_t\} \},\quad\textit{ and } \quad A_T := \sum_{t \le T} \indi\{ \widehat{Y}_t = \dk\}.
\end{align*}
As previously discussed, the performance of a learner is measured in terms of regret with respect to the best-in-hindsight abstaining classifier from $\mathcal{F}$ that makes no mistakes, that is \[ f^* \in \argmin_{f \in \mathcal{F}} \sum_{t \le  T} \indi\{f(X_t) = \dk\} \quad\textrm{s.t. }\quad  \sum_{t\le T} \indi\{f(X_t) \not\in \{\dk, Y_t\} \} = 0 . \] Note that such an $f^*$ is always realised, since the class is finite, and since it contains the all abstaining classifier. Let $A_T^* := \sum_{t \le T} \indi\{f^*(X_t) = \dk\}$ denote the value of the minimum above. The principal metrics of interest to us are the \emph{abstention regret} $A_T - A_T^*,$ and the \emph{total mistakes} $M_T$.\\

\textbf{Solution Concept} The two performance metrics naturally involve a tradeoff - for instance, making some mistakes may allow a learner to drastically reduce its abstention regret to the point that it is negative. We pursue the trade-off between the worst possible behaviour of either regret.\begin{defi} \emph{(Regret Achievability)} For functions $\phi, \psi: \mathbb{N}^2 \to \mathbb{R},$ we say that expected regret bounds of $(\phi, \psi)$ are achievable against a class of adversaries $\mathscr{C}$ if there exists a learner such that for every adversary in $\mathscr{C}$, $\mathbb{E}[A_T- A_T^*] \le \phi(T,N)$ and $\mathbb{E}[M_T] \le \psi(T,N)$. \end{defi}
As is common, we are interested in the growth rates of achievable bounds with $T$. We thus define \begin{defi} \emph{(Achievable rates)} we say that asymptotic expected-regret rates of $(\alpha, \mu) \in [0,1]^2$ are achievable against a class of adversaries $\mathscr{C}$ if an expected regret bound of $(\phi,\psi)$ can be achieved against it for functions $\phi,\psi$, said to be witnesses for the rate, such that \[ \limsup_{T \to \infty} \frac{\log \phi(T,N)}{\log T} \le \alpha \quad \textit{and} \quad \limsup_{T \to \infty} \frac{\log \psi(T,N)}{\log T} \le \mu. \] \end{defi}
Notice that if $(\alpha, \mu)$ is an achievable rate, so is $( \alpha',\mu')$ for $\alpha'\ge \alpha, \mu' \ge \mu$. As a result, the lower boundary of the set of achievable rates is well defined, and we will refer to this as the \emph{Pareto frontier of achievable rates}. This is equivalently characterised by the function $\underline{\alpha}(\mu) := \inf \{ \alpha: (\alpha, \mu) \textrm{ is an achievable rate}\}.$ This is well defined since $\forall \mu, (1,\mu)$ is achievable by always abstaining.

\section{The Adversarial Case}\label{sec:adv}

We begin with the adversarial case. The scheme, called the `versioned uniform explorer' (\textsc{vue}) is described below, and we discuss both the motivation of the scheme, and its analysis.

The main idea underlying \textsc{vue} is that any function $f$ that is observed to make a mistake on an instance $X_t$ (due to the learner abstaining on this instance) can be removed from future consideration, since we are only trying to match the behaviour of the competitor $f^*$, and clearly $f \neq f^*$ as it has made a mistake. This motivates setting up a `version space,' \[ \mathcal{V}_t:= \left\{ f:  \sum_{s < t} \indi\{Z_s\neq *, f(X_s) \not\in \{\dk, Y_s\} \} = 0 \right\},\] the set of functions that are consistent with the observations made up to time $t$. Notice that $f^* \in \vs_t$ for all $t$. Given $\vs_t$, we can restrict to playing an action in the set \smash{$\widehat{\mathcal{Y}}_t := \{ f(X_t) : f \in \mathcal{V}_t\}$} - $f^*(X_t)$ lies in this set, and thus any action outside of it can be eliminated. Of course, if \smash{$\widehat{\mathcal{Y}}_t$} is a singleton, then it contains $f^*(X_t)$, and we can just play it. 

\begin{wrapfigure}[16]{r}{0.4\textwidth}
\vspace{-\baselineskip} \begin{minipage}{0.4\textwidth}
\begin{algorithm}[H]
        \caption{\textsc{vue}}\label{alg:versioning}
        \begin{algorithmic}[1]
            \State \textbf{Inputs}: $\mathcal{F},$ Exploration rate $p$.
            \State \textbf{Initialise}: $\mathcal{V}_1 \gets \mathcal{F}$.
            \For{$t \in [1:T]$}
                \State $\widehat{\mathcal{Y}}_t \gets \left\{ f(X_t): f \in \mathcal{V}_t\right\}.$
                \If{$|\widehat{\mathcal{Y}}_t| = 1$}
                \State $\widehat{Y}_t \gets f(X_t)$ for any $f \in \mathcal{V}_t$.
                \State $\mathcal{V}_{t+1} \gets \mathcal{V}_t$.
                \Else 
                    \State Sample $C_t \sim \mathrm{Bern}(p)$.
                    \If{$C_t = 1$}
                        \State Set $\widehat{Y}_t = \dk,$ observe $Y_t$.
                        \State $\mathcal{U}_{t} \gets \{f: f(X_t) \in \{\dk,Y_t\}\}$
                        \State $\mathcal{V}_{t+1} = \mathcal{V}_t \cap \mathcal{U}_t$.
                    \Else 
                    \State Pick $\widehat{Y}_t \in \widehat{\mathcal{Y}}_t \setminus \{\dk\}$.
                    \State $\mathcal{V}_{t+1} \gets \mathcal{V}_t$.
                    \EndIf
                \EndIf
            \EndFor
        \end{algorithmic}
    \end{algorithm}
    \end{minipage}
\end{wrapfigure}
Next, since we are incentivised to minimise the total number of abstentions, it behooves us to play non-abstaining actions whenever possible. However, this puts us in a bind, since feedback is produced only when we play an abstaining action. Taking inspiration from \cite{cesa2005minimizing}, we abstain at a rate $p$ by tossing a biased `exploratory coin', $C_t$, abstaining when $C_t=  1$, and otherwise playing any non-abstaining action in $\widehat{\mathcal{Y}}_t$. Clearly, such a strategy can incur at most $pT$ excess abstention regret in expectation. Mistakes made by this strategy are controlled via the following `adversarial law of large numbers' (ALLN).
\begin{mylem}\label{lem:alln} Let $\{\mathscr{F}_t\}_{t= 1}^\infty$ be any filtration, and $\{U_t\}_{t= 1}^\infty, \{B_t\}_{t= 1}^\infty$ be $\{\mathscr{F}_t\}$-adapted binary processes, such that $B_t\sim \mathrm{Bern}(p),$ $p< \nicefrac{1}{2}$ is jointly independent of $\mathscr{F}_{t-1}, U_t$ for each $t$. Let $W_t = \sum_{s \le t} U_s,$ and $\widetilde{W}_t = \sum_{s \le t} U_s B_s$. For any $\delta \in (0, \nicefrac{1}{\sqrt{e}}),$ \[ \mathbb{P}\left( \exists t : \widetilde{W}_t \le 1 , W_t > \frac{8\log(1/\delta)}{p} \right) \le \delta.\] 
\end{mylem}

The above is argued in \S\ref{appx:alln} using a self-normalised martingale tail inequality \cite{howard2020time_uniform_chernoff}. We note that this self-normalisation is critical, and without this techniques such as Freedman's inequality yield an extraneous $\sqrt{T}$ factor in the bounds that is untenable for our purposes. The same argument, along with the shaping technique of Howard et al.~\cite{howard2018confidence_sequences} yields a Bernstein-type law of iterated logarithms that controls $|W_t - \nicefrac{\widetilde{W}_t}{p}|$ at a level $\tilde{O}(1/p + \sqrt{ W_t/p \log\log t}),$ which should be useful more broadly. This full version (presented in \S\ref{appx:alln}) further shows that the `Bernoulli-sampler' \cite{ben2020adversarial, alon2021adversarial} offers a continuous approximation in the sense of Ben-Eliezer \& Yogev \cite{ben2020adversarial}, but with the error for sets of low incidence flattened as expected due to Bernstein's inequality.

{For our purposes, the point of Lemma \ref{lem:alln} is to allow us to argue that no matter what the adversary does, if we uniformly abstain at a rate $p$, then we will `catch' any mistake-prone function before it makes ${O}(1/p)$ mistakes. Exploiting a union bound, this in turn means that with high probability, any such function will fall out of the version space $\vs_t$ before it has incurred much more than $\log N/p$ mistakes. Since the label produced by Algorithm \ref{alg:versioning} must equal $f(X_t)$ for \emph{some} $f$ in the version space, we can infer that the number of mistakes the learner makes is at most the number of times any function in the version space is wrong. Using the Lemma yields a bound of $\widetilde{O}(1/p)$ on the number of mistakes that any functions in the version space can have ever made, and since there are only $N$ possible functions, in total the number of mistakes the learner can make is bounded as $\widetilde{O}(N/p)$. More formally, the argument, presented in \S\ref{appx:adv}, argues this for a single function $f \in \mathcal{F}$ by instantiating the lemma with $\mathscr{F}_t = \sigma(\mathscr{F}_{t} = \sigma(\hist_{t}^{\mathfrak{A}}), B_t = C_t,$ and $U_t^f := 1\{f(X_t) \not\in\{\dk, Y_t\}\}$. The resulting $\widetilde{W}_t^f$ is the number of mistakes $f$ is \emph{observed} to have made, and $f \in \vs_t$ if and only if $\widetilde{W}_t^f = 0$\footnote{This argument only needs control for the case $\widetilde{W}_t = 0$. The $\le 1$ in Lemma \ref{lem:alln} is exploited in \S\ref{sec:improve_N_vue}.}. Along with a use of Bernstein's inequality to control $A_T$ this yields the result below.}

 \begin{myth}\label{thm:adv}
Algorithm \ref{alg:versioning} instantiated with $p < \nicefrac{1}{2},$ and run against an adaptive adversary, attains the following with probability at least $1-\delta$ over the randomness of the learner and the adversary: \begin{align*}
    M_T &\le \frac{9N \log(2N/\delta)}{p}\\
    A_T - A_T^* &\le pT + \sqrt{ 2p(1-p) T \log(2/\delta)} + 2\log(2/\delta).
\end{align*} 
In particular, taking $p = \sqrt{N/T}$ yields the symmetric regret bound \[ \max(M_T, A_T - A_T^*) \lesssim \sqrt{NT} \log(N/\delta).\]
\end{myth}
We conclude with a few remarks.\\

\textbf{Achievable rates} Taking $\delta = \nicefrac{1}{T},$ and varying $p$ in $(\nicefrac{\log T}{T}, 1]$ gives the rates attainable by \textsc{vue}\begin{mycor} All rates $(\alpha, \mu)$ such that $\alpha > 0, \alpha + \mu > 1$ are achievable against adaptive adversaries. \end{mycor} These rates are tight - as expressed in Corollary \ref{cor:low_bd_rates}, rates such that $\alpha + \mu < 1$ are not achievable even against stochastic adversaries. The Pareto frontier is therefore the line $\alpha + \mu = 1.$\\

\textbf{Dependence on ${N}$} It should be noted that the dependence on the number of functions, $N$, in Thm.~\ref{thm:adv} is polynomial, as opposed to the more typical logarithmic dependence on the same in online classification. The problem of characterising this dependence appears to be subtle, and we do not resolve the same. In the following section, we explore schemes that improve this aspect, but at a cost - \S\ref{sec:stoch} yields logarithmic dependence against stochastic adversaries, while \S\ref{sec:improve_N} gives a scheme that has a logarithmic dependence against adaptive adversaries, but worse dependence with $T$.

It is worth stating that the analysis above is tight for Algorithm \ref{alg:versioning} - consider the domain $\mathcal{X} = [1:N]$, and the class $\mathcal{F} = \{f_t: t \in [0:N]\}$ such that $f_t(x) = \dk$ if $x \le t$ and $=1$ if $x > t$. Now consider an adversary that chooses a $t^*$ in advance, and presents the contexts $1$ $T/N$ times, $2$ $T/N$ times and so on, labelling contexts smaller than $t^*$ as $0$, and contexts larger than $t^*$ as $1$. Notice that in each case, there is exactly one function in $\mathcal{V}_t$ that does not abstain. The scheme above incurs $\Omega(pT(1-t^*/N))$ excess abstention, and $\Omega(t^*/p)$ mistakes, and linearly large $t^*$ form a tight example. Of course, this is not a lower bound on this problem, and the question of the optimal dependence on $N$ remains open.    \\

\textbf{Hedge-Type Schemes} The natural approach of proceeding by weighing the cost of abstention versus a mistake, and running a hedge-type scheme on an importance-estimate of the resulting loss does not lead to tight rates - the scheme \textsc{mixed-loss-prod} of \S\ref{sec:improve_N} pursues precisely this strategy, and the worse case symmetric regret bounds that standard analyses lead to scale as $T^{2/3}$ instead of as $T^{1/2}$ as for \textsc{vue} (Cor.~\ref{cor:mixed_loss_prod_rates}). This may be due to the fact certain-error prone classifiers in $\mathcal{F}$ may have very low abstention rates, and thus overall large weight, and it is unclear how to eliminate this behaviour.   
\section{The Stochastic Case}\label{sec:stoch}

\begin{wrapfigure}[17]{r}{0.38\textwidth}
\vspace{-2\baselineskip}
\begin{minipage}{0.38\textwidth}
\begin{algorithm}[H]
        \caption{\textsc{vue-prod}}\label{alg:stoch}
        \begin{algorithmic}[1]
            \State \textbf{Inputs}: $\mathcal{F}, p$, Learning rate $\eta$.
            \State \textbf{Initialise}: $\mathcal{V}_1 \gets \mathcal{F}, \forall f, w_1^f \gets 1$.
            \For{$t \in [1:T]$}
                \State Sample $f_t \sim \pi_t = \frac{w_t^f\indi\{f \in \vs_t\}}{\sum_{f \in \vs_t} w_t^f}.$
                \State Toss $C_t \sim \mathrm{Bern}(p)$.
                \State $\displaystyle \pred_t \gets \begin{cases} \dk & C_t =1 \\ f_t(X_t) & C_t = 0\end{cases}.$
                \State $\vs_{t+1} \gets \vs_t$.
                \If{$C_t = 1$}
                        \State $\mathcal{U}_{t} \gets \{f: f(X_t) \in \{\dk,Y_t\}\}$
                        \State $\mathcal{V}_{t+1} = \mathcal{V}_t \cap \mathcal{U}_t$.
                    \EndIf
                \For{$f \in \mathcal{V}_{t+1}$}
                    \State $a_t^f \gets \mathds{1}\{f(X_t) = \dk\}$
                    \State $w_{t+1}^f \gets w_t^f \cdot (1-\eta a_t^f).$
                \EndFor
            \EndFor
        \end{algorithmic}
    \end{algorithm}
    \end{minipage}
\end{wrapfigure}

This section argues that the regret bounds of Thm.~\ref{thm:adv} can be improved to behave logarithmically in $N$ in the stochastic setting. There are a couple of issues with Algorithm \ref{alg:versioning} that impede a better analysis in the stochastic case. The first, and obvious, one is that how \smash{$\pred_t$} is chosen is not specified. More subtly, the fact that the scheme insists on playing non-abstaining actions whenever possible makes it difficult to control the number of mistakes without a polynomial dependence on $N$.

We sidestep these issues in Algorithm \ref{alg:stoch} by maintaining a law $\pi_t$ on functions in $\vs_t$ that only depends on $\hist_{t-1}^{\mathfrak{L}},$ and predicting by setting $\pred_t = f(X_t)$ for $f_t \sim \pi_t$. Notice that playing this way it is possible that we abstain on $X_t$ even if the exploratory coin comes up tails. We control mistakes by arguing that very error-prone functions are all quickly eliminated (due to the stochasticity), and using the property that $\pi_t$ does not depend on $X_t$ to limit the mistakes incurred up to such a time. Abstention control follows by choosing $\pi$ according to a strategy that favours $f$s with small overall abstention rate over the history. In Algorithm \ref{alg:stoch}, we use a version of the \textsc{prod} scheme of \cite{cesa-bianchi-prod} to set weights, analysed with shrinking decision sets. The following is shown along these lines in \S\ref{appx:stoch}.
\begin{myth}\label{thm:stoch}
    Algorithm \ref{alg:stoch}, run against stochastic adversaries with $\eta = p,$ attains the regret bounds\begin{align*}
        \mathbb{E}[M_T] &\le 8\frac{\log T \log(NT)}{p}, \quad \textrm{and}\quad \mathbb{E}[A_T - A_T^*] \le pT + \frac{\log N}{p}. 
    \end{align*}
\end{myth}
 We note that \textsc{vue-prod} also enjoys favourable bounds in the adversarial case - mistakes are bounded as $\tilde{O}(N/p)$, and abstention regret as in the above result. This is in contrast to simpler follow-the-versioned-leader type schemes that also satisfy similar bounds as Thm.~\ref{thm:stoch} in the stochastic case. Also note that the above cannot attain rates such that $\alpha \le \nicefrac{1}{2},$ an inefficiency introduced due to the conditional independence of $\pi_t$ and $X_t$.

Finally, we show a lower bound. The statement equates stochastic adversaries with their laws. \begin{myth}\label{thm:low_bd}
    If $\mathcal{F}$ contains two functions $f_1, f_2$ such that there exists a point $x$ for which $f_1(x) = \dk \neq f_2(x),$ then for every $\gamma \in [0,\nicefrac12],$ there exists a pair of laws $P_1^\gamma, P_2^\gamma$ such that any learner that attains $\mathbb{E}_{P_1^\gamma}[A_T - A_T^*] = K$ must incur $\mathbb{E}_{P_2^\gamma}[M_T] \ge \gamma( e^{-2\gamma K}T - K).$ 
\end{myth} Thus, if a $(\phi, \psi)$ regret bound with $\sup \frac{\phi}{T} < \frac{1}{2e^2}$ is achievable, then $\phi\cdot\psi = \Omega(T)$. Indeed, using the above with $\gamma = 1/\phi(T,N)$, gives $\mathbb{E}_{P_1}[A_T - A_T^*] = K \le \phi(T,N)$, and so $\psi(T,N) \ge \mathbb{E}_{P_2}[M_T] \ge \frac{T}{\phi(T,N)} e^{-2K/\phi(T,N)} -1.$ This proves the following. \begin{mycor}\label{cor:low_bd_rates}
    If $(\alpha, \mu) \in [0,1]^2$ is such that $\alpha + \mu < 1,$ then an $(\alpha, \mu)$ regret rate is not achievable against stochastic adversaries, and, a fortiori, against adaptive adversaries. 
\end{mycor}

\section[Reducing dependence on number of classifiers]{Reducing the dependence of regret bounds on $N$ in the adversarial case}\label{sec:improve_N}

This section concentrates on improving the $N$-dependence of regret bounds in the adversarial case via two avenues. The first improves this dependence to $\log(N)$ by running \textsc{prod} with a weighted loss, but at the cost of increasing $T$ dependence. This holds greatest relevance when $T$ is bounded as a polynomial of $N$, which is of interest because $N$ can be quite large even in reasonable settings - e.g., a discretisation of $d$-dimensional hyperplanes induces $N = \exp{Cd}$. The second approach considers the case when the set of possible contexts, i.e. $\mathcal{X}$ is not too large. While in this case, $N$ can be as large as $(|\mathcal{Y}|+1)^{|\mathcal{X}|},$ we show bounds depending only linearly on $|\mathcal{X}|.$ 

\subsection{Weighted \textsc{prod}}

\begin{wrapfigure}[13]{r}{0.34\textwidth}
\vspace{-4\baselineskip} \begin{minipage}{0.34\textwidth}
\begin{algorithm}[H]
        \caption{\textsc{mixed-loss-prod}}\label{alg:prod}
        \begin{algorithmic}[1]
            \State \textbf{Inputs}: $\mathcal{F},$ Exploration rate $p$, Learning rate $\eta$.
            \State \textbf{Initialise}: $\forall f \in \mathcal{F}, w_1^f \gets 1$.
            \For{$t \in [1:T]$}
                \State Sample $f_t \sim \pi_t = \nicefrac{w_t^f}{\sum w_t^f}.$
                \State Toss $C_t \sim \mathrm{Bern}(p)$.
                \If{$C_t = 1$}
                    \State $\pred_t \gets \dk$
                    \Else
                        \State $\pred_t \gets f_t(X_t)$
                \EndIf
                \State $\forall f \in \mathcal{F},$ evaluate $\ell_t^f$ 
                \State $w_{t+1}^f \gets w_t^f(1-\eta \ell_t^f)$.
            \EndFor
        \end{algorithmic}
    \end{algorithm}
    \end{minipage}
\end{wrapfigure}

We continue the uniform exploration, but play according to the \textsc{prod} method, with the loss \[ \ell_t^f := C_t\indi\{f(X_t) \not\in \{\dk, Y_t\} \} + \lambda \indi\{f(X_t) = \dk\}, \]  where $\lambda$ both trades-off the relative costs of mistakes and abstentions, in the vein of the fixed cost Chow loss, and accounts for the sub-sampling of the mistake loss. 

The analysis of this scheme, presented in \S\ref{appx:adv_imp_n}, exploits the quadratic bound of \textsc{prod} due to \cite{cesa-bianchi-prod} to control the sum $\mathbb{E}[pM_T + \lambda(A_T - pT)]$ by $\min_g \log N/\eta + \sum \eta (\ell_t^g)^2,$ where the expectation is only over the coins $C_t,$ and the $-pT$ term is due to the extra abstentions due to the exploratory coin. The key observation is that since $f^*$ makes no mistakes, $\sum (\ell_t^{f^*})^2 = \lambda^2 A_T^*,$ and so taking $g = f^*,$ and exploiting the weight allows us to separately control the regrets in terms of $A_T^*$.

\begin{myth}\label{thm:mixed-loss-prod}
    Algorithm \ref{alg:prod}, when run against adaptive adversaries with $\eta = \nicefrac{1}{2}, \lambda \le p$, attains \begin{align*}
        \mathbb{E}[M_T] &\le \frac{2\log N}{ p} +  \frac{2\lambda}{p}\mathbb{E}[A_T^*], \quad \textrm{and} \quad \mathbb{E}[A_T - A_T^*] \le pT + \frac{2\log N}{\lambda}.
    \end{align*}
\end{myth}

\subsubsection{Rates}

Theorems \ref{thm:stoch} and \ref{thm:mixed-loss-prod} show regret bounds with logarithmic dependence in $N$. The following concept separates rates attainable with this advantageous property from those with worse $N$-dependence.
\begin{defi}\emph{(Logarithmically Achievable Rates)} We say that rates $(\alpha, \mu)$ are logarithmically achievable against adversaries from a class $\mathscr{C}$ if there exists a learner that attains a $(\psi, \phi)$-regret against such adversaries for $\psi,\phi$ that witness the rate $(\alpha, \mu)$, and satisfy that for every fixed $T,$ $\max(\phi(T,N), \psi(T,N)) = O(\mathrm{polylog}(N)).$
\end{defi}
Since $A_T^* \le T,$ choosing $p = T^{-u}, \lambda = T^{-(u+v)}$ in \textsc{mixed-loss-prod} for any $(u,v) \in [0,1]^2, u+v\le 1$ allows us to attain rates of the form $(\alpha, \mu ) = (\max(1-u, u+v), 1-v).$ Notice that for any fixed $v$, the smallest $\alpha$ so attainable is $\nicefrac{1+v}{2}$. This shows \begin{mycor}\label{cor:mixed_loss_prod_rates}
    Any rate $(\alpha, \mu)$ such that $\alpha + \mu/2 > 1$ is logarithmically achievable against adaptive adversaries. 
\end{mycor}
The following figure illustrates the worst case achievable rate regions in the three cases considered.

\begin{figure}[!hb]
    \centering
    \includegraphics[width = .8\textwidth]{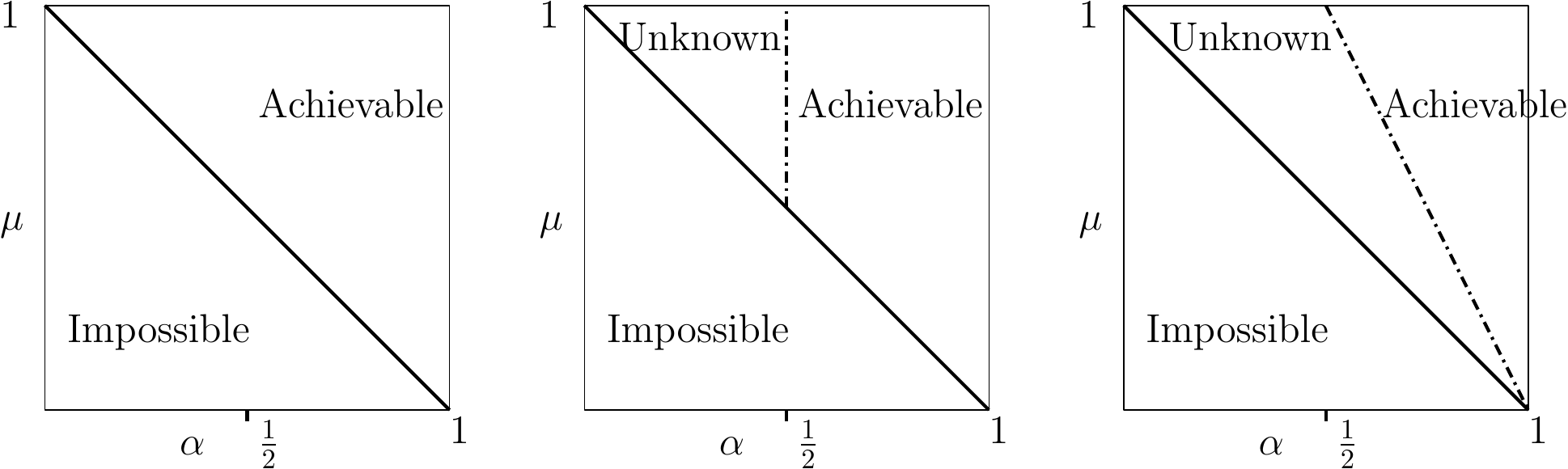}
    \caption{Left shows rates achievable against adaptive adversaries. Middle and right show logarithmically achievable rates against stochastic and adaptive adversaries respectively.}
    \label{fig:rates}
\end{figure}

\textbf{Adaptive Rates} Observe that if $A_T^* \asymp T^{\alpha^*}$ for some $\alpha^* < 1,$ then nominally, the achievable rates can be improved. Indeed, with the parametrisation $p = T^{-u}, \lambda = T^{-u+v}$, we may attain rates of the form $(\alpha, \mu) = (\max(1-u, u+v), \max(u, \alpha^* - v))$. Further, a given mistake rate $\mu$ can be attained by setting $u = \mu, $ and $\alpha^* - v \le \mu.$ With these constraints, the smallest abstention rate attainable is \[ \widetilde{\alpha}(\mu; \alpha_*) = \max\left(1 - \mu,  (1 + (\alpha^* - \mu)_+)/2\right), \] achieved by setting $v = (\alpha_* - \mu)_+, u = \min(1 - (\alpha^* - \mu)_+, 2\mu)/2.$ Such rates can in fact be attained adaptively, without prior knowledge of $\alpha^*.$ The main bottleneck here is that the quantity $A_T^*$ is not observable. However, every function $g$ that is never \emph{observed} to make a mistake satisfies $\sum(\ell_t^g)^2 = \lambda^2 \sum \indi\{g(X_t) = \dk\},$ and such functions are identifiable given \smash{$\hist^{\mathfrak{L}}_t$}. Let \[ B_t^* := \min \sum_{s \le t} \indi\{g(X_s) = \dk\} \quad \textrm{s.t.}\quad \sum_{s\le t} C_s\indi\{g(X_s) \not\in \{\dk,Y_t\} = 0. \] Note that $B_t^*$ grows monotonically, and is always smaller than $A_t^* = \sum_{s \le t} \indi\{f^*(X_t) = \dk\}$. We show the following in \S\ref{appx:adapt} via a scheme that adaptively sets $p,\lambda$ according to $B_t^*$. \begin{myth}\label{thm:mixed_loss_adaptive_rate}
    For any $\alpha^*, \mu, \epsilon \in (0,1],$ Algorithm \ref{alg:adapt} attains, without prior knowledge of $\alpha^*,$ any rate of the form $(\widetilde{\alpha}(\mu, \alpha^*) + \varepsilon, \mu+\epsilon)$ against adaptive adversaries that induce $A_T^* \le T^{\alpha^*}$ almost surely. 
\end{myth}
The rates $\widetilde{\alpha}$ essentially interpolate between the second and third panels of Fig.~\ref{fig:rates}. Concretely the region achieved consists of the intersection of the regions $\{\alpha > 1/2\}, \{\alpha + \mu > 1\}$ and $\{ 2\alpha + \mu > 1 + \alpha^*\},$ with the last set being active only when $\alpha^* \ge \nicefrac{1}{2}.$

\subsection[An instance space dependent bound]{A $|\mathcal{X}|$-dependent analysis of \textsc{vue}}\label{sec:improve_N_vue}

We give an alternate mistake analyse for \textsc{vue} over finite domains. The analysis is slightly stronger: let $\mathbf{y} \in ([1:K]\cup\{\dk\})^{|\mathcal{F}|}$ be indexed by elements of $\mathcal{F}$, with the `$f$th' entry $\mathbf{y}_f$ reprsents a value that $f$ might take. Consider the resulting partition of $\{ \mathcal{X}_{\mathbf{y}} \}_{\mathbf{y} \in ([1:K] \cup \{\dk\})^\mathcal{F}},$ where each part $\mathcal{X}_{\mathbf{y}}\subset \mathcal{X}$ contains points that have the same pattern of function values, that is $\mathcal{X}_{\mathbf{y}} = \{x: \forall f \in \mathcal{F}, f(x) = \mathbf{y}_f\}.$ The following argument can be run unchanged by replacing single $x$s in the following by all $x$s in one $\mathcal{X}_{\mathbf{y}}$. That is, we may replace $|\mathcal{X}|$ in the following Theorem~\ref{thm:adv_x} with $|\{\mathcal{X}_\mathbf{y}\}|$. For simplicity, we present the argument for $|\mathcal{X}|$ only.

Denote $\widehat{\mathcal{Y}}^x_t := \{f(x) : f \in \vs_t\}.$ Notice that after the first time $t$ such that $X_t = x, \pred_t = \dk,$ we will remove from the version space all classifiers that did not abstain or output the correct classification at time $t$. Thus if we define $y^x\in[1:K]$ to be $Y_t$, then for all subsequent times, $\widehat{\mathcal{Y}}^x_t \subset \{\dk, y^x\}$. As a result, if we observe two mistakes at any given $x$, then we cannot make any more mistakes at a subsequent time $t'$ with $X_{t'} = x,$ because the only remaining decision in $\widehat{\mathcal{Y}}^x_{t'}$ must be $\dk$. 

We may now proceed in much the same way as \S\ref{sec:adv} - instantiate $U_t^x = \indi\{X_t = x, \pred_t \not\in\{\dk, Y_t\}\},$ $B_t = C_t,$ and union bound over the $x$s. Then $|\widehat{\mathcal{Y}}_t^x| \ge 2$ if and only if $\widetilde{W}_t^x \le 1,$ and, invoking Lemma \ref{lem:alln}, up to such a time at most $W_t^x = O(\nicefrac{\log|\mathcal{X}|}{p})$ mistakes may be made on instances such that $X_t = x$. But then totting up, we make at most $O(|\mathcal{X}|\log|\mathcal{X}|/p)$ mistakes, as encapsulated below 
\begin{myth}\label{thm:adv_x}
Algorithm \ref{alg:versioning} instantiated with $p\le \nicefrac{1}{2}$ and run against an adaptive adversary, attains the following with probability at least $1-\delta$ over the randomness of the learner and the adversary: \begin{align*} 
    M_T &\le \frac{9|\mathcal{X}| \log(2|\mathcal{X}|/\delta)}{p}\\
    A_T - A_T^* &\le pT + \sqrt{ 2p(1-p) T \log(2/\delta)} + 2\log(2/\delta).
\end{align*} 
\end{myth}

Along with the bound itself, the above result makes a couple of points regarding the characterisation of $N$-dependence of the regrets in online selective classification. Firstly, it suggests that efficient analyses, and possibly schemes, must incorporate the structure of $\mathcal{X};$ and secondly it shows that constructions that attempt to show superlogarithmic in $N$ lower bounds must have both $N$ and $|\mathcal{X}|$ large, and thus typical strategies placing a very rich class on a small domain will not be effective.

\section{Experiments}\label{sec:exp}

We evaluate the performance of Algorithm \ref{alg:stoch} on two tasks - CIFAR 10 \cite{CIFAR}, and GAS \cite{GAS} - see \S\ref{appx:exp} for details of implementation, and \href{https://github.com/anilkagak2/Online-Selective-Classification}{here} for the relevant code. The former represents a setting where an expert can be adaptively invoked, which we treat by providing the true labels of the classes upon abstention. The second case is more explicitly an adaptive feature selection task - the GAS dataset has features from 16 sensors, and we train one model, $g$, on all of this data, while the selective classification task operates on data from the first 8 sensors only, and receives the output of $g$ when abstaining. The standard accuracies of the model classes we implement are $\sim 90\%$ on CIFAR-10, and $\sim 77\%$ on GAS. In both cases, a training set is used to learn a parameterized family of selective classifiers, $f_{\mu,t}$. The hyperparameters $(\mu,t)$ provide control over various levels of accuracy and abstention. For training, we leverage a recent method \cite{gangrade2021selective} that yields such a parameterisation, which is discretised to get $N=600$ of these functions to form our class $\mathcal{F}$. We then sequentially classify the test datasets of each of the tasks. 


One subtlety with the setting is that none of the selective classifiers in $\mathcal{F}$ actually make no mistakes. To avoid the trivialities emerging from this, we relax the versioning condition to only drop classifiers that are seen to make mistakes on at least  $\varepsilon N_t + \sqrt{2\varepsilon N_t}$ mistakes at time $t$, where $N_t$ is the number of times feedback was received up to time $t$, and the second term handles noise. Additionally, if it turns out that all functions in $\vs_t$ are wrong on a particular observed instance, we ignore this feedback (since such an error is unavoidable). Such variations of ‘relaxed versioning’ are natural ideas when extending the present problem to the one where the competitor may be allowed to make non-zero mistakes, although its analysis is beyond the scope of this paper. The scheme's viablility in this extended setting with only simple modifications indicates the practicality of such strategies.

{Below, we take the competitor to be the function that makes the fewest mistakes, denoted as $M_T^*$. If there is more than one such function, we take the one that makes the fewest abstention to get $A_T^*$. We measure \emph{excess mistakes} $M_T - M_T^*$ and excess abstentions $A_T - A_T^*$ with respect to this competitor.}\\

\textbf{Behaviour of regrets with the length of the game} Fig.~\ref{fig:vary_T} presents the excess mistakes as a fraction of $T$ for the two datasets, i.e.~$\nicefrac{M_T - M_T^*}{T},$ as $T$, is varied. The learners are all instantiated with the exploration rate $p = \nicefrac{1}{\sqrt{T}}$. We observe that the excess abstentions are negative (or near-zero) over this range (see Fig.~\ref{fig:abs_vary_T} in \S\ref{appx:exp}). Therefore we do not plot these below (the orange line is MMEA, see below). We note that the relative mistakes stay below $\sqrt{\nicefrac{2\log N}{T}},$ bearing out the theory.\\ 

\textbf{Achievable Operating Points of Mistakes and Abstentions} Fig.~\ref{fig:vary_p} shows the mistake and abstention rates attainable by varying $p$ and $\epsilon$, while holding $T$ fixed at $500$ (which is large enough to show long-run structure, but small enough allow fast experimentation). Concretely, we vary these linearly for $20$ values of $p \in [0.015, 0.285],$ and $10$ values of $\epsilon \in [0.001,0.046].$ The resulting values represent operating points that can  be attained by a choice of $p,\epsilon$. The same plot includes lines that represent the operating points when the scheme is run with $\epsilon = 0.001,$ the smallest value we take. Note that in practice, the best choices of $\epsilon,p$ may be data dependent, and choosing them in an online way is an interesting open problem (also see \S\ref{appx_sensitivity_wrt_eps}). \\

\textbf{The Price of Being Online} We characterise this in two ways beyond the excess mistakes. \begin{itemize}[nosep, topsep = 0pt]
    \item In Fig.~\ref{fig:vary_T}, we also plot the `mistake-matched excess abstention' (MMEA). This is defined as follows - if the scheme concludes with having made $M_T$ mistakes, we find, in hindsight, the classifier that minimises the number of abstentions, subject to making at most $M_T$ mistakes. The MMEA is the excess abstention of the learner over those of this relaxed competitor, and represents how many fewer abstentions a batch learner would make if allowed to make as many mistakes as the online learner. Notice that this MMEA remains well controlled in Fig.~\ref{fig:vary_T}, and appears to scale as \smash{$O(\sqrt{T})$.}
    \item In Fig.~\ref{fig:vary_p},  we also plot the post-hoc operating points of the classifiers in $\mathcal{F}$ as black triangles. This amounts to plotting the optimal abstentions amongst classifiers that make at most $m$ mistakes, varying $m$.\footnote{Observe that the MMEA corresponds to the horizontal distance between a red-point with $m$ mistakes, and the left-most black point with $y$-coordinate under $m$.} We note that the red operating points of the scheme get close to the black frontier, illustrating that the inefficiency due to being online is limited. As the time-behaviour of MMEA in Fig.\ref{fig:vary_T} illustrates, the inefficiency is expected to grow sublinearly with $T$, and to thus vanish under amortisation. 
\end{itemize} 
\begin{figure}[ht]
    \centering
    \includegraphics[width = 0.45\textwidth]{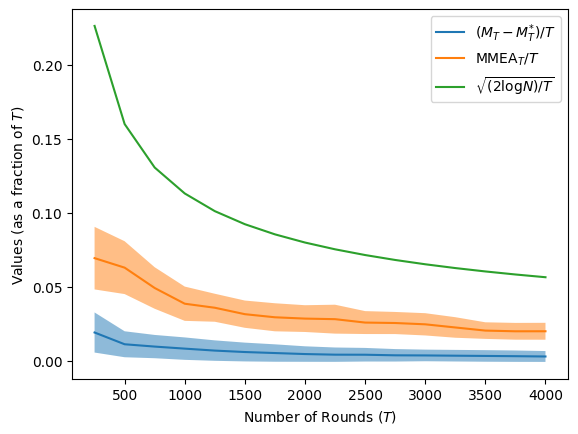}~\includegraphics[width=0.45\textwidth]{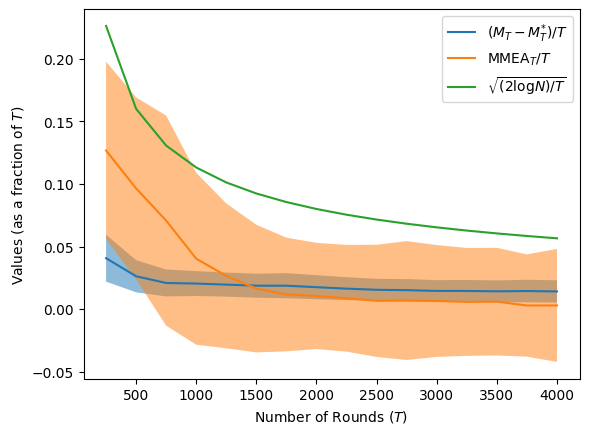}
    \caption{\small{$M_T-M_T^*$, and MMEA as fractions of $T$, as the number of rounds $T$ is varied for CIFAR-10 (left) and GAS (right). The plots are averaged over $100$ runs, and one-standard-deviation error regions are drawn. }}
    \label{fig:vary_T}
\end{figure}
\begin{figure}[ht]
    \centering
    \includegraphics[width = 0.45\textwidth]{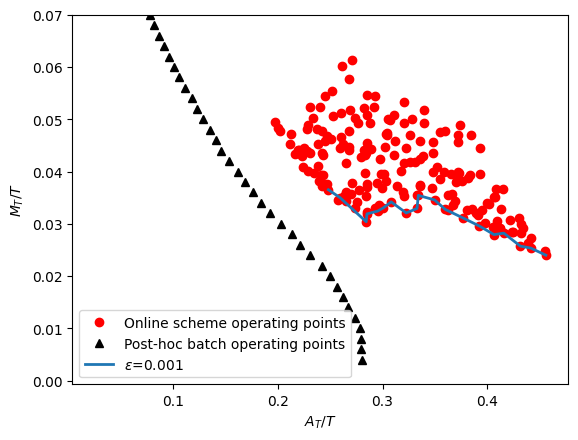}~\includegraphics[width=0.45\textwidth]{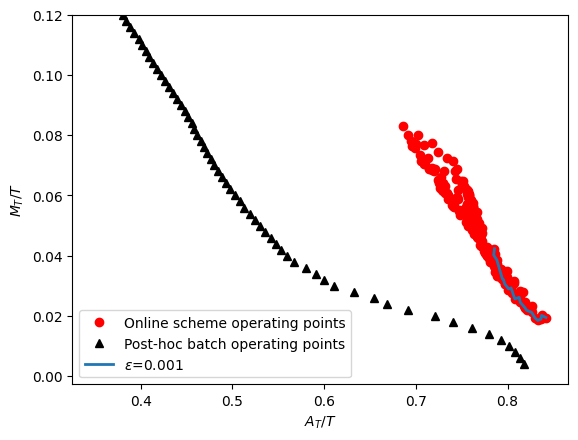}
    \caption{\small{Operating points for our scheme as $\epsilon$ and $p$ are varied are represented as red dots (for CIFAR-10 in the left, and GAS in the right). The black triangles represent operating points obtained by batch learning with the benefit of full feedback. The blue lines interpolate points obtained by varying $p$ for $\epsilon = 0.001$ Points are averaged over 200 runs. Note that the values are raw mistakes and abstentions, and not regrets.}}
    \label{fig:vary_p}
\end{figure}

\section{Discussion}

Online selective classification offers a primitive that has relevance to both safety-critical and resource-limited settings. In the paper, we highlighted the role of long-term abstentions in such situations, and studied this problem under the feedback limitation that labels are only provided when the system abstains, which is the only time high-complexity evaluation would be invoked in a selective classification system. When working with a finite class of model, we identified a simple scheme that provides a tight (in terms of $T$) trade-off between mistakes and excess abstentions against adaptive adversaries. We further discussed two schemes that improve upon the dependence of the same on the size of the model class - tightly against stochastic adversaries, and at the cost of some rate performance against adaptive adversaries. Together, these schemes and analyses provide some basic foundations for the problem when competing against no-mistake models. Additionally, we carried out empirical studies that validate the scheme in the stochastic case, and demonstrate that with minor modifications, the scheme is resilient to the situation where no selective classifier in the model class is mistake-free. A number of interesting questions remain open, and we discuss a few of these below. 

Perhaps the most basic question left open by the above study is how the minimax regrets against adaptive adversaries depend on $N$. Along with being a basic scientific question, this issue has implications for whether the results can be extended to infinite classes. Indeed, under assumptions of bounded combinatorial dimensions, the \textsc{vue-prod} and \textsc{mixed-loss-prod} schemes can be extended to infinite model classes, but the basic technique to do so yields trivial bounds for \textsc{vue} due to the linear dependence on $N$. If this dependence could be improved to logarithmic, the extension to model classes with finite (multiclass versions of) Littlestone dimension would be immediate.

A practically relevant and theoretically interesting direction is online SC but where the competitor can make non-zero mistakes. This can be set up in at least two ways - either an error parameter $\epsilon$ is given to the learner, which must ensure that both notions of regret are small against competitors that make at most $\epsilon T$ mistakes; or, no explicit error parameter is specified, and the learner is required to compete against the least mistake-prone model in a given set (similarly to \S\ref{sec:exp}). Both settings raise new challenges, since one must relax the notion of versioning used in the above work for related schema to be viable. The latter setting raises a further issue of how one can adapt to the mistake rate of the competitor. Also of practical relevance is the case where abstentions are not equally penalised, but have some variable cost. Here too, one can study variants of signalling regarding whether the cost of abstention is available before or only after an abstaining decision is made.

Finally, we observe that while tight, the random exploration technique is somewhat unsatisfying, and practically a context-adapted abstention strategy is likely to offer meaningful advantages over it. In analogy with the exploration in label-efficient prediction, one direction towards exploring context-aware methods is to study more concrete structured situations, such as linear models with noisy feedback that are popular in the investigation of online selective sampling.

\paragraph*{Acknowledgements} Our thanks to Tianrui Chen for helpful discussions.

\paragraph*{Funding Disclosure} This research was supported by the Army Research Office Grant W911NF2110246, the National Science Foundation grants CCF-2007350 and CCF-1955981, ARM Research Inc, and the Hariri Data Science Faculty Fellowship Grants. Additional revenues related to this work: AC was a visiting researcher at Google when this work was completed.

\printbibliography

\clearpage

\begin{appendix}

\section{An Adversarial Anytime Uniform Law of Large Numbers For Probing Binary Sequences}\label{appx:alln}

\subsection{Proofs of Lemma\ref{lem:alln}}

We begin with a simple lemma that underlies the remaining argument. Below, $\kappa$ is chosen so that $\kappa''(0) = 1$. \begin{mylem}\label{lem:self_norm}
    Let $\mathscr{F}_t, U_t,B_t, W_t, \widetilde{W}_t$ be as in Lemma \ref{lem:alln}. Let $\bp = 1-p$. Then for any $\eta \in \mathbb{R},$  the process \[\xi_t^\eta := \exp{\eta (W_t - \widetilde{W}_t/p) - \kappa(\eta) V_t} \] is a non-negative, $\mathscr{F}_t$-adapted martingale, where \begin{align*}
        V_t &= \frac{\bp}{p} W_t, \\
        \kappa(\eta) &= \frac{p}{\bp} \log\left(p e^{-\eta \bp/p} + \bp e^\eta \right).
    \end{align*}
        
\begin{proof}
    The nonnegativity of $\xi_t^\eta$ is trivial, and it is $\mathscr{F}_t$-adapted since it is a deterministic function of the adapted processes $W_t, \widetilde{W}_t$. We need to argue that $\xi$ is a martingale. To this end, observe that since $W_t = \sum_{s < t} U_s, \widetilde{W}_t = \sum_{s < t} U_s B_s,$ \[ \xi_t^{\eta} = \xi_{t-1}^\eta \cdot \exp{\eta U_t (1-B_t/p - \bp\kappa(\eta)/p) }.\]
    Due to the independence of $B_t$ from $\sigma(U_t, \mathscr{F}_{t-1}),$ we have \begin{align*}
        &\phantom{=} \mathbb{E}[ \exp{\eta U_t (1-B_t/p)}|\mathscr{F}_{t-1},U_t]\\ &= \left( p e^{ -\eta U_t \bp/p} + \bp e^{\eta U_t} \right)\\
        &\overset{*}= \left( p e^{ -\eta \bp/p)} + \bp e^{\eta} \right)^{U_t} = \exp{ \frac{\bp}{p} U_t \kappa(\eta)},
    \end{align*} 
    where the equality marked $*$ exploits the fact that $U_t$ is $\{0,1\}$-valued. Rearranging, we have \[\mathbb{E}\left[\exp{\eta U_t (1-B_t/p) - \frac{\bp}{p} U_t\kappa(\eta)}\middle|\mathscr{F}_{t-1},U_t\right] = 1, \] and exploiting the tower rule, we conclude that \[ \mathbb{E}[\xi_t^{\eta}|\mathscr{F}_{t-1}] = \xi_{t-1}^{\eta} \mathbb{E}\left[ \mathbb{E}\left[ \exp{\eta U_t (1-B_t/p) - \frac{\bp}{p} U_t\kappa(\eta)}\middle|\mathscr{F}_{t-1},U_t\right] \middle| \mathscr{F}_{t-1}\right] = \xi_{t-1}^\eta. \qedhere \] 
\end{proof}
\end{mylem}

The following argument heavily exploits the techniques of Howard et al. \cite{howard2020time_uniform_chernoff}, and assumes familiarity with the same. It also exploits the property that only the upper tail of $\Delta_t$ is being controlled, although this is extended in the following section.

\begin{proof}[Proof of Lemma \ref{lem:alln}]
    
We define the deviation of $W_t$ from $\widetilde{W}_t$ as \[ \Delta_t := W_t - \frac{\widetilde{W}_t}{p}.\] Notice that $\Delta_0 = 1.$ As a result of the above lemma, $\Delta_t$ is a 1-sub-$\kappa$ process with the associated variance process $V_t$, in the sense of Definition $1$ of Howard et al.~\cite{howard2020time_uniform_chernoff}. In particular, since $\kappa$ is the (normalised) cumulant generating function of a centred Bernoulli random variable taking values $\{-\bp/p, 1\}$, the process is sub-binary. Further, since $p < \nicefrac{1}{2}, \bp/p > 1,$ and thus the process is sub-gamma, with the scale parameter $c = 0.$ \cite[\S3.1, and Prop.2]{howard2020time_uniform_chernoff}. 

We can thus invoke the line-crossing inequality of Corollary 1, part c) of Howard et al., instantiated with $c = 0$ to find that for any $x, m > 0$ \[ \mathbb{P}\left(\exists t: \Delta_t \ge x + \mathfrak{s}(x/m)(V_t -m)\right)\le \exp{- \frac{x^2}{2m}},\] where \cite[Table 2]{howard2020time_uniform_chernoff} \[ \mathfrak{s}(x/m) = \frac{x}{2m}.\] 

Plugging these in, we observe that \[ \mathbb{P}\left(\exists t: \Delta_t \ge \frac{x}{2}  + \frac{x}{2m}V_t\right) \le \exp{ -\frac{x^2}{2m}}.\]

Now notice that if $V_t \ge m,$ then $x/2 + (x/2m) V_t \le (x/m)V_t$. Therefore, we can conclude that \[ \P\left(\exists t: \Delta_t \ge \frac{x}{m}V_t, V_t \ge m \right) \le \exp{-\frac{x^2}{2m}},\] and substituting $V_t = \frac{\bp}{p} W_t, \Delta_t = W_t - \nicefrac{\widetilde{W}_t}{p},$ \[\mathbb{P}\left( \exists t: \frac{\widetilde{W_t}}{p} \le \frac{mp - x\bp}{pm} W_t, W_t \ge \frac{pm}{\bp} \right) \le \exp{- \frac{x^2}{2m}}. \]

Now, if we choose \( m = \frac{\bp}{p} (x + \nicefrac{1}{p}) \) it follows that \[ \forall W_t \ge \frac{p}{\bp}m,  \frac{pm - x\bp}{mp} W_t \ge \frac{1}{p},\] and thus \[ \mathbb{P}\left( \exists t: \frac{\widetilde{W}_t}{p} \le \frac{1}{p}, W_t \ge \frac1p+x\right) \le \exp{-\frac{px^2}{2(\nicefrac1p + x)\bp}}, \] and choosing $x \ge 1/p$ further ensures that \[ \mathbb{P}\left( \exists t: \widetilde{W}_t\le1, W_t \ge 2x \right) \le \exp{-\frac{px}{4\bp}}.\] Now, setting $x = \max\left( \frac{1}{p}, \frac{4\bp}{p}\log(1/\delta) \right)$ leaves us with \[ \mathbb{P}\left( \exists t: {\widetilde{W}_t} \le 1, W_t \ge \max\left(\frac{2}{p}, \frac{8 \bp}{p} \log(1/\delta)\right) \right) \le \delta. \] The conclusion follows on observing since $p < \nicefrac{1}{2}, 8\bp \ge 4,$ and thus, for $\log(1/\delta) \ge1/2,$ $\frac{2}{p} \le 8\frac{\bp}{p}\log(1/\delta)$. \end{proof}

\subsection{An improved ALLN via a Self-Normalised Law of Iterated Logarithms}

The line-crossing inequalities we utilised in the previous subsection can be stitched together, by picking an exponentially increasing set of $x$s, and optimising the $m$s at each, to yield a curve crossing inequality, which in effect determines a curve that the deviations are unlikely to cross. We use the results of Howard et al.~\cite{howard2018confidence_sequences} that produce non-asymptotic constructions. 

For our purposes, note that the processes $\Delta_t$ and $-\Delta_t$ are both sub-Gamma with variance process $V_t$, with the scale parameters $c_+ = 0$ and $c_- = \frac{1}{3} \cdot \frac{1-2p}{p}$ respectively. The former property is useful for controlling the upper deviations of $\Delta_t$, and the latter for the lower deviations. Note that since the scale parameter $c_+$ is 0, the upper tails in the following can be improved, but for ease of presentation we will just set $c = |c_+| = c_-$ in the following. 

Using Theorem 1 of Howard et al. \cite{howard2018confidence_sequences} twice - for $\Delta_t$ and $-\Delta_t$, and instantiating it with $\eta = e, h(k) = \frac{\pi^2k^2}{6}$ yields that for the sub-gamma process $\Delta_t$ with scale parameter $\le c$, and variance process $V_t,$ and any constant $m > 0,$ and for the functions \begin{align*}S_{m, \delta}(v) &= 2\sqrt{v \ell_{m, \delta}(v) } + c \ell_{m,\delta}(v),\\ \ell_{m,\delta}(v) &= \log \frac{\pi^2}{6} + 2\log\log \frac{v}{m} + \log \frac{2}{\delta},\end{align*} the following bound holds true \[\mathbb{P}(\exists t: |\Delta_t| \ge \mathcal{S}_{m, \delta}( \max(V_t,m) ) \le \delta.\]

The curve $S(\max(V_t,m))$ can be simplified upon observing that \[\{\exists t: V_t \ge m,  |\Delta_t| \ge \mathcal{S}_{m,\delta}(V_t)\} \subset \{\exists t: |\Delta_t| \ge \mathcal{S}_{m,\delta}(\max(V_t,m))\}.\] With the above in hand, set $m = \bp/p,$ so that $W_t \ge 1 \iff V_t \ge m,$ and observe that $\log(\pi^2/6) < 1.$ The following bound is immediate upon recalling that $V_t = \frac{\bp}{p} W_t, \Delta = W_t - \nicefrac{\widetilde{W}_t}{p}$. \begin{myth}In the setting of Lemma \ref{lem:alln}, 
\[ \mathbb{P}\left( \exists t: W_t \ge 1, |W_t - \nicefrac{\widetilde{W}_t}{p}| \ge 2\sqrt{ \frac{\bp W_t}{p} \left( \log \frac{2e}{\delta} + 2\log\log W_t \right)} + \frac{\log\nicefrac{2e}{\delta} + 2\log\log W_t}{3p}  \right) \le \delta.\] 
\end{myth} 

Technically, the $\log\log W_t$ is not always defined in the above. This should be read as $\log(\max(1,\log W_t))$ to handle edge cases - alternately, it can be handled by replacing $W_t \ge 1$ by $W_t \ge 3 > e$ in the above.\footnote{In a similar vein of edge-cases, if $W_t < 1\implies W_t = 0,$ then $0 \le \widetilde{W}_t \le W_t = 0,$ and thus the bound extends to all possible values of $W_t$.}

Notice that the bound above has the correct form when taking into account the behaviour of binomial tails, which $\widetilde{W}_t$ behaves like. Indeed, if $W$ is some natural number valued random variable, and $\widetilde{W}|W \sim \mathrm{Bin}(W,p),$ then Bernstein's inequality \cite[Ch.~2]{boucheron2013concentration} states that \[\mathbb{P}\left( |W - \widetilde{W}/p| \ge C\sqrt{ \bp\frac{W}{p} \log(2/\delta)} + C\log(2/\delta) \right) \le \delta,\] which entirely parallels the form of the above theorem, barring the $\log\log W_t$ blowup due to the uniformity over time. 

The above analysis was inspired by studying the recent work of Ben-Eliezer and Yogev \cite{ben2020adversarial}, on adversarial sketching - their goal was to maintain an estimate of the incidence of a process within a given set (and more generally, within sets in a given system) while using limited memory, and they analysed a similar sampling approach, showing via an application of Freedman's inequality that \cite[Lemma 4.1]{ben2020adversarial} \[ \mathbb{P}\left( |W_T - \widetilde{W}_T/p| \ge C\sqrt{\frac{T}{p} \log(2/\delta)} + C\frac{\log(2/\delta)}{p}\right) \le \delta.\] This essentially amounts to using the crude bound $W_T \le T$. The same paper, in Theorem 1.4 and associated lemmata argues that the Reservoir Sampler \cite[\S2]{ben2020adversarial} of size $\sim pT$ controls deviations uniformly over time at scale $\sqrt{\frac{T}{p} \log\frac{\log T}{\delta}},$ and it was asserted that the Bernoulli Sampler cannot attain such a `continuous robustness'\cite[\S1]{ben2020adversarial}. The above result improves upon this in a few ways - firstly, the result applies to the simpler Bernoulli sampler, and improves the deviation control to $O(\sqrt{W_t})$ instead of $O(\sqrt{T})$. This has the further advantage that if one is concerned with the number of samples queried along with the memory, the Bernoulli sampler only queries $\sim pT$ times with high probability, while the reservoir sampler queries about $pT \log T$ times. Secondly, it shows that the Bernoulli sampler \emph{does} offer continuous robustness, but up to a flattening of the deviation control for sets of small incidence (small $W_t$). Ben-Eliezer \& Yogev show a number of applications of such bounds to sketching, and Alon et al.~have recently applied this to tightly characterise the regret in online classification \cite{alon2021adversarial}, using techniques of Rakhlin et al. \cite{rakhlin2015online, rakhlin2015sequential}. We believe that self-normalised bounds as above can contribute to showing adaptive versions of these results.

\section{Analysis of \textsc{vue} Against Adaptive Adversaries} \label{appx:adv}

This section serves to show Theorems \ref{thm:adv} and \ref{thm:adv_x}. We will analyse the excess abstention, and the mistakes separately. Both deviations are controlled with probability $1-\delta/2$, and so a union bound completes the argument. The excess abstention control is common to both, and exploits Bernstein's inequality. 

\begin{proof}[Proof of excess abstention bound]
    Notice that the procedure only abstains if $C_t = 1$ or if $\widehat{\mathcal{Y}}_t = \{\dk\}$. In the latter case, the competitor also abstains, and thus no excess abstention is incurred. Therefore, the net excess abstention is bounded as $A_T - A_T^* \le \sum C_t$. Now, $\sum C_t$ is a Binomial random variable with parameters $T, p$. By Bernstein's inequality \cite[Ch.~2]{boucheron2013concentration}, \[ \mathbb{P}\left( \sum C_t \ge pT + 2\sqrt{p(1-p) T \log(2/\delta)} + 2\log(2/\delta) \right) \le \frac{\delta}{2}.\qedhere \]
\end{proof}

We move on to bounding mistakes in a $N$-dependent way.
\begin{proof}[Proof of mistake bound from Theorem \ref{thm:adv}]

As in the main text, consider the filtration $\{\mathscr{F}_t\} = \{\sigma(\hist_t^{\mathfrak{A}})\}$, $U_t^f := \indi\{f(X_t) \not\in \{\dk, Y_t\} \},$ and consider the processes $W_t^f = \sum_{s < t} U_t^f, B_t = C_t, \widetilde{W}_t^f = U_t^f C_t.$ Note that since $N \ge 2, \frac{\delta}{2N} \le \frac{1}{4} \le \frac{1}{\sqrt{e}}.$

Note that for every $f$, $U_t^f$ and $C_t$ satisfy the requirements of Lemma \ref{lem:alln}, since $C_t$ is tossed independently of $\hist_{t-1}^{\mathfrak{A}}.$ Therefore, we may invoke Lemma \ref{lem:alln} to find that \[ \mathbb{P}\left(\exists t: \widetilde{W}_t^f =0, W_t^f \ge \frac{8}{p} \log(\nicefrac{2N}{\delta})\right) \le \frac{\delta}{2N}, \] and applying a union bound over $f \in \mathcal{F},$ we conclude that \[ \mathbb{P}\left(\exists t, f: \widetilde{W}_t^f =0, W_t^f \ge \frac{8}{p} \log(\nicefrac{2N}{\delta})\right) \le \frac{\delta}{2}, \]

Notice that if $\widetilde{W}_{t-1}^f$ is non-zero, then $f \not\in \vs_t$ since we've seen it make a mistake prior to the time $t$. Now define the stopping times \( \tau_f := \max\{t : f \in \mathcal{V}_t\} = \max\{t:\widetilde{W}_{t-1}^f = 0\}.\) We observe that \begin{align*}
    M_T &= \sum_t \indi\{\err{t}\} \le \sum_t \indi\{\exists f \in \mathcal{V}_t: f(X_t) \not\in \{\dk, Y_t\} \}\\
        &\le \sum_f \sum_t \indi\{ f \in \vs_t, f(X_t) \not\in \{\dk, Y_t\} \}\\
        &= \sum_f \sum_t \indi\{t \le \tau_f\} U_t^f.
\end{align*}
Next, define the event \[ \mathsf{E} := \left\{ \exists t, f : f \in \mathcal{V}_t, W_{t-1}^f \ge 8\log(2N/\delta)/p\right\}.\] 

Since $f \in \vs_t \iff \widetilde{W}_{t-1}^f = 0 \iff t \le \tau_f$. Also recall that $W_{t-1}^f = \sum_{s < t} \indi\{f(X_s) \not\in \{\dk,Y_s\}$. Therefore, given $\mathsf{E}^c$, \[ \sum_{t}  \indi\{ t \le \tau_f, f(X_t) \not\in \{\dk, Y_t\}\} \le 8\frac{\log(2N/\delta)}{p} + 1,\] since on $\mathsf{E}^c,$ $t \le \tau_f \implies \widetilde{W}_{t-1}^f = 0 \implies \sum_{s < t} U_t^f \le \frac{8 \log(\nicefrac{2N}{\delta})}{p},$ and the additional $1$ arises since $\mathsf{E}^c$ does not control behaviour at $\tau_f$. We conclude that given $\mathsf{E}^c,$ we have \[    M_T  \le \sum_f 9 \frac{\log(2N/\delta)}{p} = 9 \frac{N \log(2N/\delta)}{p}. \] But $\mathsf{E}$ occurs with probability at most $\delta/2,$ and we have shown that \[ \mathbb{P}\left( M_T > \frac{9N\log(\nicefrac{2N}{\delta})}{p}\right) \le \frac{\delta}{2}. \qedhere \]

\end{proof}

As discussed in \S\ref{sec:improve_N}, the $\mathcal{X}$-dependent argument proceeds similarly. \begin{proof}[Proof of mistake bound from Theorem \ref{thm:adv_x}]

Again, consider the filtration $\{\mathscr{F}_t\} = \{\sigma(\hist_t^{\mathfrak{A}})\}$. Define $\widehat{\mathcal{Y}}_t^x = \{ f(x): f \in \vs_t\}$, and the process $U_t^x := \indi\{X_t = x, \pred_t \not\in\{\dk, Y_t\}\},$ and consider the processes $W_t^x = \sum_{s < t} U_t^x, B_t = C_t, \widetilde{W}_t^x = U_t^x C_t.$ Again, since $|\mathcal{X}| \ge 2, \frac{\delta}{2|\mathcal{X}|} \le \frac{1}{4} \le \frac{1}{\sqrt{e}}.$

Invoking Lemma \ref{lem:alln}, since $C_t$ is tossed independently of $\hist_{t-1}^{\mathfrak{A}},$ we find that \[ \mathbb{P}\left(\exists t: \widetilde{W}_t^x \le 1, W_t^x \ge \frac{8}{p} \log(\nicefrac{2|\mathcal{X}|}{\delta})\right) \le \frac{\delta}{2|\mathcal{X}|}, \] and applying a union bound over $x \in \mathcal{X},$ we conclude that \[ \mathbb{P}\left(\exists t, x: \widetilde{W}_t^x \le 1, W_t^x \ge \frac{8}{p} \log(\nicefrac{2|\mathcal{X}|}{\delta})\right) \le \frac{\delta}{2}, \]

Now, from the argument in the main text, $U_t^x \ge 0 \implies |\widehat{\mathcal{Y}}_t^x| \ge 2 \iff W_{t-1}^x \le 1.$ So, define the stopping times \[\tau_x := \max\{t : |\widehat{\mathcal{Y}}_t^x| \ge 2\} = \max\{t: W_{t-1}^x \le 1\}. \]

We have that \begin{align*}
    M_T &= \sum_t \indi\{\err{t}\} \\
        &= \sum_x \sum_t \indi\{ |\widehat{\mathcal{Y}}_t^x| \ge 2\} U_t^x \\
        &= \sum_x \sum_t \indi\{t \le \tau_x\} U_t^x.
\end{align*}
Defining the event \[ \mathsf{E} := \left\{ \exists t, x : t \le \tau_x, W_{t-1}^x \ge 8\log(\nicefrac{2|\mathcal{X}|}{p})\right\},\] 
we again observe that given $\mathsf{E}^c,$ \[ \sum_{t}  \indi\{ t \le \tau_x\} U_t^x \le 1 +  8\frac{\log(2N/\delta)}{p},\] since on $\mathsf{E}^c,$ $t \le \tau_x \iff \widetilde{W}_{t-1}^x \le 1 \implies \sum_{s \le t-1} U^x_{s} \le \frac{8\log(\nicefrac{2|\mathcal{X}|}{\delta}}{p}$. We thus conclude that\[    M_T  \le \sum_x 9 \frac{\log(2|\mathcal{X}|/\delta)}{p} =  \frac{9|\mathcal{X}| \log(2|\mathcal{X}|/\delta)}{p}. \] But $\mathsf{E}$ occurs with probability at most $\delta/2,$ and we have shown that \[ \mathbb{P}\left( M_T > \frac{9|\mathcal{X}| \log(\nicefrac{2|\mathcal{X}|}{\delta})}{p}\right) \le \frac{\delta}{2}. \qedhere \]

\end{proof}
\section{Stochastic Adversaries} \label{appx:stoch}

This section contains proofs omitted from \S\ref{sec:stoch}.

\subsection{Performance of \textsc{vue-prod}}
This section consitutes a proof of Theorem \ref{thm:stoch}. We begin by controlling the excess abstentions.

\begin{proof}[Proof of excess abstention bound]
    
    We begin by analysing the \textsc{prod} algorithm for the setting where decision sets may shrink with time. For succinctness, denote $a_t^f = \indi\{f(X_t) = \dk\}, A_t^f := \sum_{s \le t} a_t^f.$
    \begin{mylem}\label{lem:shrinking_prod}
        Let $\pi_t^f$ be as in Algorithm \ref{alg:stoch}. If $\eta \le \nicefrac{1}{2},$ then for any $g \in \vs_T,$ it holds that \[ \sum_{t,f} \pi_t^f a_t^f \le \frac{\log N}{\eta} + A_T^g + \eta \sum_{t \le T} (a_t^g)^2.\]
        \begin{proof}
            We follow the standard analysis of \textsc{prod}, updated slightly to account for versioning. Consider the potential $W_t := \sum_{f \in \vs_t} w_t^f,$ where recall that $w_t^f = \prod_{s < t} (1-\eta a_s^f)$. Since the weights are always non-negative, for any $g \in \vs_T,$ we have that \[ W_{T+1} \ge \prod_{t \le T} (1-\eta a_t^g). \] Therefore, we have the lower bound \begin{align*}
                \log \frac{W_{T+1}}{W_1} \ge -\log N + \sum \log (1-\eta a_t^g) \ge -\log N - \sum \eta a_t^g - \sum (\eta a_t^g)^2,  
            \end{align*}  
            which exploits the fact that for $z \le \nicefrac{1}{2,} \log(1-z) \ge -z - z^2.$
            
            To upper bound the same quantity, notice that for any $t$, \[ W_{t+1} = \sum_{f \in \vs_{t+1}} w_{t+1}^f \le \sum_{f \in \vs_{t}} w_{t}^f (1-\eta a_t^f) = W_{t} \left(1- \eta\sum_{f} \pi_{t}^f  a_t^f\right),\] which again exploits that weights are non-negative, and that $\vs_t$ is a non-increasing sequence of sets. Taking ratios and bounding $\log(1-z)$ by $-z$, and finally summing over $t = 1:T,$ we have \[ \log \frac{W_{T+1}}{W_1} = \sum_t  \log \frac{W_{t+1}}{W_{t}} \le -\eta \sum_t \sum_{f} \pi_{t}^f a_t^f. \] Rearranging the inequality obtained by sandwiching $\log \frac{W_{T+1}}{W_1}$ yields the bound.
        \end{proof}
    \end{mylem}
    
    Note that the above lemma holds generically, for any loss $\ell_t^f \le 1,$ and any sequence of shrinking decision sets. We will exploit this fact later. 
    
    For our purposes, observe that since $a_t^f$ is an indicator, $(a_t^f)^2 = a_t^f$. Thus, using Lemma \ref{lem:shrinking_prod} for $g = f^* \in \vs_T,$ \[ \sum_{t,f} \pi_t^f a_t^f \le \frac{\log N}{\eta} + A_T^* + \eta A_T^*. \] 
    Now, the total abstention incurred by the learner is \[ A_T = \sum \indi\{C_t = 1\} + \indi\{C_t = 0, f_t(X_t) = \dk\}. \] Exploiting the independence of the exploratory coin, we find that \[\mathbb{E}[A_T] = pT + (1-p) \mathbb{E}[ \sum_{t,f} \pi_t^f a_t^f]. \] Invoking the above bound on $\sum_{t,f} \pi_t^f a_t^f$ and rearranging then yields that  \[ \mathbb{E}[A_T] \le pT + \frac{(1-p) \log N}{\eta} + (1-p) \mathbb{E}[A_T^*] + \eta(1-p) \mathbb{E}[A_T^*].\] Now, if $\eta = p$, then $\eta(1-p) -p = -p^2 < 0$, and then exploiting that $A_T^* \ge 0$ yields the bound \[ \mathbb{E}[A_T - A_T^*] \le pT + \frac{\log N}{p}.\qedhere \]
    
    This leaves the mistake control. The argument we present critically relies on the law $\pi_t^f$ being chosen independently of $X_t$, given $\hist^{\mathfrak{L}}_{t-1}$. This is ultimately a source of inefficiency - for instance, if $\pi_t^f$ were allowed to depend also on $X_t,$ then we could enforce that non-abstaining actions are not played when $C_t= 0$, and drop the second $\log(N)/p$ term from the excess abstention bound. However, we were unable to show mistake control with only logarithmic dependence on $N$ in this situation.
    \begin{proof}[Proof of mistake bound]
        The mistake control proceeds by partitioning the class $\mathcal{F}$ according to the mistake rates of individual $\mathcal{F}s$ and arguing that whole groups of these are simultaneously, and quickly, eliminated from the version space without incurring too many mistakes. This fundamentally exploits the stochasticity of the setting. 
        
        To this end, define 
        \begin{align*}
            \mathcal{F}_\zeta &:= \{ f \in \mathcal{F}: 2^{-\zeta} \le P(f(X_t) \not \in \{?,Y_t\}) \le 2^{1-\zeta}\\
            \overline{\mathcal{F}}_\zeta &:= \{ f \in \mathcal{F} : P(f(X_t) \not\in \{?, Y_t\} ) \le 2^{-\zeta}\}.
        \end{align*}

        In the following, $\zeta_0$ is a parameter for the purposes of analysis, that will be chosen later. Notice that $\mathcal{F} = \bigcup_{\zeta \le \zeta_0} \mathcal{F}_\zeta \cup \overline{\mathcal{F}}_{\zeta_0}.$
        
        We'll argue that all $f \in \mathcal{F}_\zeta$ are eliminated quickly (for small $\zeta$). For this, it is useful to define the stopping times \[\tau_\zeta := \max \{t : \exists f \in \mathcal{F}_\zeta \cap \vs_t \}. \]Notice that for any $f \in \mathcal{F}_\zeta,$ \[ P( C_t = 1, f(X_t) \not\in \{\dk, Y_t\}) \ge 2^{-\zeta} p.\] As a consequence of this and the union bound, we have the following tail inequality. \begin{mylem} For any $\delta \in (0,1)$,
            \[ \mathbb{P}\left( \exists \zeta \le \zeta_0: \tau_\zeta > \sigma_{\delta, \zeta_0}(\zeta) \right) \le \delta, \] where \[ \sigma_{\delta,\zeta_0}(\zeta) := \frac{2^{\zeta}}{p} \log(\zeta_0 N/\delta).\]
        \end{mylem}
        
        With this in hand, notice that \begin{align*}
            M_T &= \sum_t \sum_f \indi\{f_t = f\} \indi\{f(X_t) \not\in\{\dk, Y_t\} \\
                &= \sum_t \sum_{\zeta\le \zeta_0} \sum_{f \in \mathcal{F}_\zeta}  \indi\{f_t = f\} \indi\{f(X_t) \not\in \{\dk, Y_t\} \} + \sum_t \sum_{f \in \overline{\mathcal{F}}_{\zeta_0}} \indi\{f_t = f\} \indi\{f(X_t) \not\in \{\dk, Y_t\}\}.
        \end{align*} 
        Next, we observe that \begin{align*} \mathbb{E}\left[ \sum_{f \in \mathcal{F}_{\zeta}} \indi\{f_t = f\} \indi\{f(X_t) \not\in \{\dk, Y_t\} \middle| \hist_{t-1}^{\mathfrak{L}}\right] & = \sum_{f \in \mathcal{F}_\zeta} \pi_t^f P(f(X_t) \not\in\{\dk, Y_t\}) \\
        &\le 2^{1-\zeta} \pi_t( f_t \in \mathcal{F}_{\zeta}) \\
        &\le 2^{1-\zeta} \mathds{1}\{t \le \tau_\zeta\},\end{align*}
        where the first equality  is because $\pi_t^f$ is predictable given $\hist^{\mathfrak{L}}_{t-1},$ the second uses the definition of $\mathcal{F}_{\zeta}$, and the final inequality is because $\pi_t$ is a distribution that is supported on $\vs_t$, and thus has total mass at most $1$, and mass $0$ when $\mathcal{F}_\zeta \cap \vs_t = \varnothing$. In much the same way, also notice that \[ \mathbb{E}\left[\sum_{f \in \overline{\mathcal{F}}_{\zeta_0}} \mathds{1}\{f_t = f, f(X_t) \not\in\{\dk, Y_t\}\} \middle| \hist^{\mathfrak{L}}_{t-1} \right] \le 2^{-\zeta_0}.\]
        
        Exploiting both the linearity of expectations and the tower rule, \begin{align*}
            \mathbb{E}[M_T] &\le \sum_t \sum_{\zeta \le \zeta_0} 2^{1-\zeta} P(\tau_\zeta \ge t) + 2^{-\zeta_0} T \\
                            &\le \sum_{\zeta \le \zeta_0} \left( 2^{1-\zeta} \sum_{t \le \sigma_{\delta, \zeta_0}(\zeta)} 1 +  \sum_{t > \sigma_{\delta, \zeta_0}(\zeta)}\delta\right) + 2^{-\zeta_0} T\\
                            &\le 2 \zeta_0 \frac{\log(\zeta_0 N/\delta)}{p} + 2\delta T + 2^{-\zeta_0} T. 
        \end{align*}
        
        Now set $\zeta_0 = \lfloor \log T\rfloor, \delta = \nicefrac1T.$  Since $\zeta_0 N /\delta \le N^2 T^2,$ we find that \[ \mathbb{E}[M_T] \le 4\frac{\log T \log(NT)}{p} + 4,\] and finally since $p \le 1,$ $\frac{4}{p} \ge 4,$ leading to the claimed bound (for $T \ge 3$). 
    \end{proof}
\end{proof}

\subsection{Lower Bound}

\begin{proof}[Proof of Theorem \ref{thm:low_bd}]

Without loss of generality, assume $f_2(x) = 1.$ Recall that $f_1(x) = \dk$.
We describe the two adversaries - \begin{itemize}
    \item $P_1^{\gamma}$ is supported on $\{(x, 1)\},$ so that for each time $X_t = x,$ and the label $Y_t = 1$.
    \item $P_2^{\gamma}$ is supported on $\{(x,1), (x,2)\}$ such that for each time $X_t = x$, while the label is drawn iid from the law $Y_t = \begin{cases} 1 & \textrm{w.p. } 1-\gamma \\ 2 & \textrm{w.p. } \gamma \end{cases}.$
\end{itemize}

Notice that against $P_1^\gamma,$ the competitor is $f_2$, which attains $A_T^{(P_1^\gamma)} = 0,$ while against $P_2^\gamma,$ the competitor is $f_1$, which attains $A_T^{(P_2^{\gamma})} = T$. Observe further that since $\gamma < \nicefrac{1}{2},$ if any learner does not play $\dk,$ it is advantageous for it to play $1$ and never play $2$.\footnote{More formally, given any leaner, we can create the better---in expectation---learner that abstains when the given one does, and predicts $1$ when the given one plays something other than $\dk$.} We thus lose no generality in assuming that the learner's actions lie in $\{\dk, 1\}.$ Now, run two coupled versions of the learner, so that if these observe the same $Z_t$s, they produce identical actions. Feed the first of these data generated from $P_1^\gamma,$ and the second of these data generated from $P_2^\gamma$. 

Let $\eta_1$ be the (random) number of abstentions that the first version of the learner makes - this means that it must have played $1$ $T-\eta_1$ times. Denote the number of mistakes that the second version of the learner makes as $\eta_2$. Given $\eta_1,$ the second version gets exactly the same sequence as the first with probability $(1-\gamma)^\eta_1$  - indeed, due to the coupling, they first abstain together, and then receive the same label with probability $1-\gamma$. Conditioned on this, they again abstain together, and then receive the same label with probability $1-\gamma$ and so on, $\eta_1$ times. This means that, given $\eta_1,$ and the event that they get the same sequence, the second version of the learner plays $T-\eta_1$ `$1$' actions. Since each of these is wrong with probability $\gamma,$ independently and identically, \[ \mathbb{E}[\eta_2|\eta_1] \ge (1-\gamma)^{\eta_1} \gamma (T-\eta_1). \]
Notice that $(1-\gamma)^{\eta_1} $ is a convex function of $\eta_1$. Thus, $\mathbb{E}[ (1-\gamma)^\eta_1] \ge (1-\gamma)^{\mathbb{E}[\eta_1]} = (1-\gamma)^K.$ Further, $\mathbb{E}[-(1-\gamma)^{\eta_1} \eta_1] \ge \mathbb{E}[-\eta_1] = -K,$ and finally, for $\gamma \le \nicefrac{1}{2}, (1-\gamma) \ge e^{-2\gamma}$. It follows that \[ \mathbb{E}[\eta_2] \ge (1-\gamma)^K \gamma T - \gamma K = \gamma (e^{-2\gamma K} T - K). \qedhere\]

While here, let us also comment that the proof of Corollary \ref{cor:low_bd_rates} is mildly incomplete, since the argument requires that $\phi \ge 2.$ If instead $\phi < 2,$ then notice that setting $\gamma = 1/2$ in the above, and using that $\mathbb{E}[\eta_2] \ge \gamma((1-\gamma)^K T - K),$ we have $\psi \ge 2^{-\phi}\frac{T}{2} - \frac{2}{2} \ge \frac{T}{8}-1,$ which grows linearly with $T$. 
    
\end{proof}
\section{Analysis of \textsc{mixed-loss-prod} Against Adaptive Adversaries}\label{appx:adv_imp_n}

This section provides a proof of Theorem \ref{thm:mixed-loss-prod}, and describes an adaptive variant of the same scheme, based on a doubling trick, that serves to show Theorem \ref{thm:mixed_loss_adaptive_rate}.

\begin{proof}[Proof of Theorem \ref{thm:mixed-loss-prod}]
    Recall that the scheme runs \textsc{prod} with the loss \[\ell_t^f := \indi\{C_t = 1\} \indi\{f(X_t) \not\in \{\dk, Y_t\} \} + \lambda \indi\{f(X_t) = \dk\}.\] We first observe that repeating the proof of Lemma \ref{lem:shrinking_prod} with $a_t^f$ replaced by $\ell_t^f$ gives us that for any $g \in \vs_T,$ \begin{equation}\label{ineq:prod_main}
        \sum_{t,f} \pi_t^f \ell_t^f \le \frac{\log N}{\eta} + \sum_t \ell_t^g + \eta \sum (\ell_t^g)^2.\end{equation}
    
    Note that this relation holds given the context and label processes. For $g = f^* \in \vs_T,$ we observe that $ \ell_t^{f^*} = \lambda \indi\{f^*(X_t) = \dk\}, $ since by definition $f^*$ makes no mistakes. Instantiating the above with $f^*,$ and noting $\sum \indi\{f^*(X_t) = \dk\} = A_T^*,$ we conclude that \begin{equation}\label{ineq:prod} \sum_{t,f} \pi_t^f \ell_t^f \le \frac{\log N}{\eta} + \lambda A_T^* + \eta \lambda^2 A_T^*.\end{equation}
    
    We proceed to characterise the mistakes and abstentions that the learner makes in terms of $\sum_{t,f} \ell_t^f$. To this end, notice that \begin{align*}
        M_T = \sum_{t,f} \indi\{f_t = f\} \cdot \indi\{C_t = 0\} \cdot \indi\{f(X_t) \not\in \{\dk, Y_t\} \}.
    \end{align*} 

        As a result, integrating over the randomness of the algorithm, but not over the contexts or labels, we find that \begin{align*}
       \mathbb{E}[M_T] &= \mathbb{E}\left[ \sum_{t,f} \mathbb{E}[ \indi\{f_t = f\} \indi\{C_t = 0\} \indi\{f(X_t) \not\in\{\dk,Y_t\} \} |\hist_{t-1}^\mathfrak{A}, X_t, Y_t]  \right] \\
                                &= \sum_{t,f} \mathbb{E}\left[ \pi_t^f (1-p) \indi\{f(X_t) \not\in \{\dk, Y_t\} \right].
    \end{align*} 

    But, observe that \begin{align} \mathbb{E}[\pi_t^f \ell_t^f] &= \mathbb{E}\left[ \mathbb{E}[ \pi_t^f C_t \indi\{f(X_t) \not\in\{\dk, Y_t\} + \lambda \pi_t^f \indi\{f(X_t) = Y_t\} |\hist_{t-1}^{\mathfrak{A}}] \right] \notag\\
                                                         &=\mathbb{E}[p\pi_t^f \indi\{f(X_t) \not \in \{\dk, Y_t\}\}] + \lambda \mathbb{E}[\pi_t^f \indi\{f(X_t) = \dk\}].\notag\end{align} 
     Therefore, \begin{equation}\label{eqn:temp} \mathbb{E}[M_T] =  \sum_{t,f} \mathbb{E}\left[\frac{(1-p)}{p} \left(\pi_t^f \ell_t^f - \pi_t^f \lambda \indi\{f(X_t) = \dk\}\right) \right]. \end{equation}
    
    Further, notice that \[ A_T = \sum_t \indi\{C_t = 1\} + \sum_{t,f} \indi\{C_t = 0\}\indi\{f_t = f\} \indi\{f(X_t) = \dk\},\] and thus, \[ \mathbb{E}[A_T] = \mathbb{E}\left[pT + (1-p) \sum_{t,f}  \pi_t^f \indi\{f(X_t) = \dk\}\right].\] Moving the negative terms in (\ref{eqn:temp}) to the left hand side, and exploiting the above, we find that \begin{equation*} \mathbb{E}[M_T] + \frac{\lambda}{p}\mathbb{E}[A_T - pT] = \frac{1-p}{p} \mathbb{E}\left[ \sum_{t,f}  \pi_t^f \ell_t^f \right],\end{equation*} where we note that both the terms $\mathbb{E}[M_T]$ and $\mathbb{E}[A_T-pT]$ are non-negative. 
    
    Exploiting the inequality \ref{ineq:prod} and the above relation, we conclude that \begin{equation}\label{eqn:mixed_prod_master} \mathbb{E}[M_T] + \mathbb{E}\left[\frac{\lambda}{p} (A_T - pT)\right] \le   \mathbb{E}\left[ \frac{1-p}{p}\left( \frac{\log N}{\eta} + \lambda A_T^* + \eta \lambda^2 A_T^* \right) \right].\end{equation}
    
    The required bounds are now forthcoming. Dropping the $M_T$ term in the left hand side of (\ref{eqn:mixed_prod_master}), and pushing the constants $N, \eta, p, \lambda$ through the expectations, \begin{align*}
        \frac{\lambda}{p} \mathbb{E}[A_T - pT] &\le  \frac{(1-p)\log N}{p\eta} + \frac{(1-p)\lambda}{p} \mathbb{E}[A_T^*] + \frac{\eta (1-p) \lambda^2}{p} \mathbb{E}[A_T^*] \\
        \iff \mathbb{E}[A_T - pT] &\le \frac{(1-p) \log N}{\eta \lambda} + (1 - p)\mathbb{E}[A_T^*] + \eta \lambda (1-p) \mathbb{E}[A_T^*] \\
        \iff \mathbb{E}[A_T - A_T^*] &\le pT + \frac{\log N}{\eta \lambda} + (\eta \lambda - p) \mathbb{E}[A_T^*].
    \end{align*}  
    Taking $\eta = \nicefrac{1}{2}, \lambda \le p,$ observe that the last term is negative (since $A_T^* \ge 0$). Thus, making these substitutions and dropping the final term gives the required excess abstention control.
    
    In a similar way, dropping the $\mathbb{E}[A_T - pT]$ term in (\ref{eqn:mixed_prod_master}) gives \begin{align*}
        \mathbb{E}[M_T] &\le \frac{\log N}{p\eta} + \frac{\lambda (1 + \eta \lambda)}{p} \mathbb{E}[A_T^*].
    \end{align*}
    The claim follows on setting $\eta = \nicefrac{1}{2},$ and observing that $\eta \lambda \le 1$.
\end{proof}

\subsection[Adapting to small competitor abstention]{Adapting Rates for small $A_T^*$}\label{appx:adapt}

\subsubsection[The form of the optimal rate]{Deriving the form of $\widetilde{\alpha}$}

We first describe a derivation of the form of $\widetilde{\alpha}$. As noted, the relevant parametrisation is $p = T^{-u}, \lambda = T^{-(u+v)},$ for $u, v \ge 0$. This, with the bounds of the previous section gives the control \begin{align*} \mathbb{E}[M_T] &\le 2 T^u  \log N + T^{\alpha^* - v} \\
\mathbb{E}[A_T - A_T^*] &\le T^{1-u} + 2T^{u+v} \log N + T^{\alpha^* - u - v}. \end{align*}

Notice that $\alpha^* - u -v \le 1 -  u - v \le 1-u,$ since $\alpha^* \le 1, v \ge 0.$ Thus, we have the rate bounds \begin{align*}
    \mu &= \max(u, \alpha^* - v) \\
    \alpha &= \max(1-u, u+v)
\end{align*} 

Deriving the optimal $\alpha$ attainable for a fixed $\mu$ then amounts to the following convex program \begin{align*}
    \min &\max(1-u, u+v) \\
     \textrm{s.t. } & 0 \le u \le \mu\\
                   & \max(0, \alpha^* - \mu) \le v 
\end{align*} 

Notice that the objective is a non-decreasing function of $v$, so the optimal choice of the same is $(\alpha^* - \mu)_+,$ the smallest value it may take. This leaves us with trying to minimise $\max(1-u , u + (\alpha^* - \mu)_+)$ for $0\le u \le \mu$. The unconstrained minimum of this function occurs at $u_0 = \frac{1 - (\alpha^* - \mu)_+}{2},$ which is feasible if $\mu \ge u_0$. If on the other hand $\mu < u_0,$ then the max-affine function is in the decreasing branch $1-u,$ and the optimal choice of $u$ is just $\mu$. Thus, the optimum is achieved at \begin{align*}
    v &= (\alpha^* - \mu)_+ \\
    u &= \begin{cases} \frac{1 - (\alpha^* - \mu)_+}{2} & 1 - (\alpha^* - \mu)_+ \le 2\mu \\ \mu & 1 - (\alpha^* - \mu)_+ > 2\mu \end{cases} = \frac{\min(1- (\alpha^* - \mu)_+, 2\mu)}{2}.
\end{align*} 

Correspondingly, $\widetilde{\alpha}$ takes the form \[ \widetilde{\alpha}(\mu; \alpha^*) = \begin{cases} \frac{1 + (\alpha^* - \mu)_+}{2} & 1 - (\alpha^* - \mu)_+ \le 2\mu \\ \max(1-\mu, \mu + (\alpha^* - \mu)_+) & 1 - (\alpha^* - \mu)_+ > 2\mu\end{cases}.\] But, \[ 1 -(\alpha^* -\mu)+ > 2\mu \iff 1-\mu \ge \mu + (\alpha^* - \mu)_+,\] and therefore \[ \widetilde{\alpha}(\mu; \alpha^*) = \begin{cases} \frac{1 + (\alpha^* - \mu)_+}{2} & 1 - (\alpha^* - \mu)_+ \le 2\mu \\ 1-\mu & 1 - (\alpha^* - \mu)_+ > 2\mu\end{cases} = \max\left( 1 -\mu, \frac{1 + (\alpha^* - \mu)_+}{2}\right).\]
\subsubsection{Adaptive Scheme and Proofs}
We start by recalling the definition of $B_t^*$ \[ B_t^* = \min_{f \in \vs_t} \sum_{s \le t} \indi\{f(X_t) = \dk\}. \] We will also use the term \[ \beta_t^* := \frac{\log B_t^*}{\log T}.\]

For the remainder of this section, let $\kappa:= \frac{\lambda}{p}$. Recall that the optimal behaviour is attained by setting $p = T^{-u}, \kappa = T^{-v},$ where \begin{align*}
    u &= \frac{\min(1- (\alpha^* - \mu)_+, 2\mu)}{2} \\
    v &= (\alpha^* - \mu)_+.
\end{align*} Algorithm \ref{alg:adapt} essentially consitutes a doubling trick by setting $p $ and $\kappa$ in phases, which are indexed by non-negative integers, $n$. The scheme is parametrised by a scale parameter, $\theta$.

\begin{itemize}
    \item We begin in the zeroth phase, with $\kappa = 1, p = T^{-\min(1,2\mu)/2}$ This phase ends when $\beta^*$ first exceeds $\mu$, at which point the first phase begins.
    \item At the beginning of each phase, we re-initialise the scheme. 
    \item For $n\ge 1,$ the $n$th phase ends when (the reinitialised) $\beta^*$ first exceeds $\mu + n \theta$. 
    \item Each time the $n$th phase ends, we restart the scheme, with $\kappa = T^{-(n+1)\theta}, p = T^{- \min(1-(n+1)\theta, 2\mu)/2}.$
\end{itemize}

Since the scheme is restarted in each phase, we may analyse each phase separately. Note that if $A_T \le T^{\alpha^*}$ almost surely, then the index of the largest phase is at most $n^* = \lfloor \frac{(\alpha^* - \mu)_+}{\theta} \rfloor $ phases, since $\beta_t^* \le \alpha^*$ always. For convenience, we set $T_n$ to be the length of the $n$th phase. Times $t_n$ correspond to rounds within the $n$th phase, and $M_{T_n}^n, A_{T_n}^n$ are the number of mistakes and abstentions incurred by the learner in the $n$th phase, while , $A_{T_n}^{*,n}$ is the number of abstentions incurred by $f^*$ in the $n$th phase. 

Consider the behaviour in the $n$th phase. Let $g_n$ be the function that minimises $\sum_{s_n \le T_n} \indi\{g(X_t) = \dk\},$ subject to $\sum_{s_n \le T_n} C_t \indi\{g(X_t) \not\in \{\dk, Y_t\} = 0$, and set the value of this optimum to $B_{T_n}^{*,n}$  By exploiting inequality $(\ref{ineq:prod_main})$ instantiated with $g_n$, and setting $\eta = \nicefrac12,$ we may infer that \[ \sum_{t_n \le T_n} \pi_{t_n}^f \ell_{t_n}^f \le 2{\log N} + p_n \kappa_{n} B_{T_n}^{*,n} +  \frac{p^2_n \kappa_n^2}{2} B_{T_n}^{*,n}.\]

As a result, reiterating the previous analysis over the $n$th phase, the number of mistakes and abstentions incurred in this phase \begin{align*}
    \mathbb{E}[M_{T_n}^{n}] &\le \frac{2\log N}{p_n} + 2\mathbb{E}[\kappa_n B_{T_n}^{*,n}] \\
    \mathbb{E}[A_{T_n}^{n} - B_{T_n}^{*,n} ] &\le \mathbb{E}[p_nT_n + 2\frac{\log N}{\kappa_n p_n}]
\end{align*}

Further, notice that in each phase, $B_{T_n}^{*,n} \le T^{\mu + (n+1)\theta}$, $\kappa_n = T^{-n\theta}, p_n = T^{-\min(1- n\theta, 2\mu)/2}$. Substituting these into the above bounds, we have \begin{align*}
    \mathbb{E}[M_{T_n}^{n}] &\le 2T^{\min(1 - n\theta, 2\mu)/2} \log N + 2T^{\mu + \theta} \le 4 T^{\mu + \theta} \log N \\
    \mathbb{E}[A_{T_n}^{n} - B_{T_n}^{*,n}] &\le T^{-\min(1 - n\theta, 2\mu)/2} \mathbb{E}[T_n] + T^{n\theta + \nicefrac{\min(1 - n\theta, 2\mu)}{2}} \log N 
\end{align*}

But then, summing over the phases, \begin{align*} \mathbb{E}[M_T] &= \sum_{0\le n \le n^*} \mathbb{E}[M_{T_n}^{n}] \\
&\le 4T^{\mu}\log N \cdot (n^* + 1) T^{\theta}\\
&\le 4T^{\mu} \log N \cdot \frac{T^{\theta}}{\theta}. \end{align*} Further, 
\begin{align*} \mathbb{E}[A_T-A_T^*] &= \mathbb{E}[\sum_{n \le n^*} A_{T_n}^n - A_{T_n}^{*,n}] \\
                       &\le \mathbb{E}[\sum_{0 \le n \le n^*} A_{T_n}^n - B_{T_n}^{*,n}] \\
                       &\le  \mathbb{E}[\sum_{0 \le n \le n^*} T^{-\min(\mu, \nicefrac{1-n\theta}{2})} T_n] + \log N \sum_{0 \le n \le n^*} T^{n\theta + \min( \nicefrac{1-n\theta}{2}, \mu)} \\
                       &\le \left(\sum_{n = 0}^{n^*} T^{1-\min(\mu, \nicefrac{1-n\theta}{2})} + \sum_{n = 0} T^{n\theta + \min( \nicefrac{1-n\theta}{2}, \mu)}\right)\log N.\end{align*}                        

To simplify the above, let $n_0 = \lfloor \frac{1-2\mu}{\theta} \rfloor$, so that $\min(\mu, \frac{1-n\theta}{2}) = \mu$ for $n \le n_0.$ Notice that $n_0$ may be bigger or smaller than $n^*$. We can then write the bound as \begin{align*}
    \frac{\mathbb{E}[A_T - T_T^*]}{\log N} &\le \sum_{n = 0}^{\min(n^*, n_0)} T^{1-\mu} + \sum_{n = \min(n^*, n_0) + 1}^{n^*} T^{\frac{1 + n\theta}{2}} + \sum_{n = 0}^{\min(n^*, n_0)} T^{n\theta + \mu} + \sum_{n = \min(n^*, n_0)+1}^{n^*} T^{\frac{1 + n\theta}{2}},
\end{align*}

where we interpret $\sum_{n = i}^j = 0$ for $i > j$. This can further be simplified to \begin{align*}
    \frac{\mathbb{E}[A_T - A_T^*]}{\log N} &\le \min(n^*+1, n_0+1) T^{1-\mu} + \frac{T^\mu}{T^\theta - 1} T^{(\min(n^*, n_0) +1)\theta)} + 2\mathbf{1}\{n_0 < n^*\} \frac{ T^{\frac{1 + (n^* + 1)\theta}{2}}}{T^{\theta/2} - 1}. 
\end{align*}

If we further assume that $\theta$ is chosen so that $T^{\theta/2} \ge 2,$ we can lower bound $T^{\theta/2} - 1 \ge T^{\theta/2}/2, T^{\theta} -1 \ge T^{\theta}/2$ which gives the bound \[ \frac{\mathbb{E}[A_T - A_T^*]}{4\log N} \le (\min(n_0,n_*) + 1) \left( T^{1 - \mu} + T^{\mu + \min(n_0,n^*) \theta} + \indi\{n_0 < n^*\} T^{ (1 + n^* \theta)/2} \right), \]

from which we can derive the rate control \[ \alpha \le \zeta(\mu, n_0, n^*, \theta) = \max(1-\mu, \mu + \min(n_0, n^*) \theta, \indi\{n_0 < n^*\} (1 + n^*\theta)/2)\]

The exact statement of the theorem is now straightforward to prove 
                
\begin{proof}[Proof of Theorem \ref{thm:mixed_loss_adaptive_rate}]
    
    We run the above procedure with $\theta = \frac{2\ln 2}{\log T}$. Notice that $T^{\theta/2} \ge 2,$ and that $T^{\theta}/\theta \le \frac{2}{\ln 2} \log T \le T^{\varepsilon}$ for large enough $T$. Therefore, mistakes are controlled at $O(T^{\mu + \epsilon})$.
    
    Further, for the abstention control, again $\min(n^*, n_0) + 1 \le n_0 + 1 \le \frac{1}{\theta} = \frac{\log T}{2\ln 2}$. Recall the abstention rate bound $\zeta$ above. It suffices to argue that $\zeta \le \widetilde{\alpha} + \theta,$ since $T^\theta = 4 = O(1).$
    
    To this end, first notice that \[ n_0 < n^* \iff \lfloor \frac{1 - 2\mu}{\theta} \rfloor  < \lfloor \frac{(\alpha^* - \mu)_+}{\theta} \rfloor \implies 1-2\mu < (\alpha^* - \mu)_+.\] In this case, \begin{align*} \zeta &= \max\left(1-\mu, \mu + n_0 \theta ,\frac{ 1 + n^*\theta}{2} \right)  \\
    &\le \max\left(1 - \mu, \mu + \frac{(1-2\mu)}{\theta} \cdot \theta , \frac{1 + \frac{(\alpha^* - \mu)_+}{\theta} \cdot \theta}{2}\right) \\
    &= \max\left(1-\mu, \frac{1 + (\alpha^* - \mu)_+}{2}\right) \\
    &= \widetilde{\alpha}(\mu;\alpha^*).\end{align*}
    
    On the other hand, if $n_0 \ge n^*$ then we have that \[ \frac{(\alpha^* - \mu)_+}{\theta} -1 \le \frac{(1-2\mu)}{\theta} \iff \mu \le \frac{1 + \theta - (\alpha^*-\mu)_+}{2}. \] As a result, in this case, \begin{align*} \zeta &\le \max\left( 1 - \mu, \mu + n^* \theta\right)\\
    &\le \max\left( 1-\mu, \mu + (\alpha^* - \mu)_+\right) \\
    &\le \max\left(1 - \mu, \frac{1 + (\alpha^* - \mu)_+ + \theta}{2}\right)\\
    &\le \widetilde{\alpha}(\mu;\alpha^*) + \theta/2 \qedhere\end{align*}
\end{proof}

\begin{algorithm}[t]
        \caption{\textsc{Adaptive-mixed-loss-prod}}\label{alg:adapt}
        \begin{algorithmic}[1]
            \State \textbf{Inputs}: $\mathcal{F},$ Time $T$, Mistake rate $\mu$, Scale $\theta$.
            \State \textbf{Initialise}: $n\gets 0; n_{\max} \gets \lceil 1/\theta\rceil; \forall f \in \mathcal{F}, w_1^f \gets 1; \forall n \le n_{\max}, \tau_n \gets T.$ 
            \For{$t \in [1:T]$}
                \State $u \gets \min(1 - n\theta, 2\mu)/2, v \gets n\theta$
                \State $p \gets T^{-u}, \lambda \gets T^{-(u+v)}$.
                \State Sample $f_t \sim \pi_t = \nicefrac{w_t^f}{\sum w_t^f}$.
                \State Toss $C_t \sim \mathrm{Bern}(p)$.
                \If{$C_t = 1$}
                    \State $\pred_t \gets \dk$
                    \Else
                        \State $\pred_t \gets f_t(X_t)$
                \EndIf
                \State $\forall f \in \mathcal{F},$ evaluate \[ \ell_t^f = C_t \indi\{f(X_t) \not\in \{\dk, Y_t\}\} + \lambda  \indi\{f(X_t) = \dk\} \] 
                \State $w_{t+1}^f \gets w_t^f(1-\eta \ell_t^f)$.
                \State Compute \begin{align*}
                    B^{*} &= \min_{g\in \mathcal{F}} \sum_{\tau_n < s \le t} \indi\{g(X_s) = \dk\}\\
                          &\phantom{=}\textrm{s.t. } \sum_{\tau_n < s \le t} C_s\indi\{g(X_s) \not\in \{\dk, Y_t\} = 0.
                \end{align*} 
                \If{$\log B^* \ge (\mu + n\theta) \log T$}
                    \State $n \gets n+1$
                    \State $\tau_{n+1} \gets t$
                    \State $\forall f \in \mathcal{F}, w_{t+1}^f \gets 1.$
                \EndIf
            \EndFor
        \end{algorithmic}
    \end{algorithm}

\section{Details of Experiments.}\label{appx:exp}

N.B.~ Code required to reproduce the experiments is provided at \url{https://github.com/anilkagak2/Online-Selective-Classification}.

\subsection{Dataset Details}
GAS \cite{GAS} dataset is a $6$-way classification task based on the $16$ chemical sensors data. These sensors are used to discriminate $6$ gases at various levels of concentrations. The data consists of these sensor readings for over a period of $36$ months divided into $10$ batches. There are $13,910$ data points in this dataset. We use the first $7$ batches as training set and the remaining $3$ batches as test set. This split results in train and test sets with $9546$ and $4364$ data points respectively. The gas task contains data from $16$ sensors (each of which gives $8$ numbers). The standard error attained by the class we use (see below) on this is $\approx 87\%$. For the selective classification task, we use only the data from the first $8$ sensors (and thus only $64$ out of $128$ features). The standard error attainable for this is $\approx67\%$. Importantly, for the GAS task, the selective classification setting we study only demands matching the performance of the best classifier with the full $16$-sensor data, and thus supervision for the $8$-sensor function is according to this best function. To be more concrete, denote the training data as $\{(X^1_i, X^2_i, Y_i)\},$ where $X^1$ and $X^2$ are the features from the first and second $8$ sensors respectively, and $Y$ is the label. We train a classifier $g$ on this whole dataset. Then we produce the labelled dataset $\{(X^1_i, g(X^1_i, X^2_i) )\},$ and train selective classifiers on this dataset. The online problem then takes the test dataset, and gives to the learner only the $X^1$ features from it. If the learner abstains, then the label $Y_t = g(X_t^1, X_t^2)$ is given to the learner.
 
CIFAR-10 \cite{CIFAR} dataset is a popular image recognition dataset that consists of $32\times 32 $ pixels RGB images of $10$ classes. It contains $50,000$ training and $10,000$ test images. We use standard data augmentations (shifting, mirroring and mean-std gaussian normalisation) for preprocessing the datasets. The best standard error attainable for this task by the models we use (see below) is $\approx 90\%$. This experiment is more straightforward to describe- selective classifiers are trained on the whole dataset. For the online problem, the test image is supplied to the learner, and if it abstains, then the true label of that image is provided as feedback.
 
\subsection{Training Experts}
\cite{gangrade2021selective} proposed a scheme to train classifiers with an in-built abstention option. This scheme provides a loss function, which takes a single hyper-parameter $\mu$, and is trained as a minimax program using gradient ascent-descent. The scheme then uses the outputs of this training with a second hyper-parameter $t$ to provide classification or abstention decisions. Therefore, the scheme utilises two hyper-parameters $(\mu, t)$ to control the classification accuracy and abstentions.  

We trained selective classifiers using this scheme. As per their recommendation, we used $30$ values of $\mu$ with $10$ values equally spaced in $[0.01, 1]$ and remaining $20$ values in the $[1, 16]$. For the threshold parameter $t$, we used $20$ equally spaced values in $[0.2, 0.95)$. The minimax program was run with the learning rates ($10^{-4}, 10^{-6}$) for the descent and ascent respectively. Notice that the resulting set of classifiers have $20 \times 30 = 600$ functions.

Note that classification on CIFAR-10 is a relatively difficult task than GAS. Hence, we used a simpler 3-layer fully connected neural network architecture for the GAS dataset, and a Resnet32 architecture \cite{resnet_implementation, resnet}  for the CIFAR-10 dataset. 

\subsection{Algorithm implementation, Hyper-parameters, Compute requirements}
We implemented Algorithm~\ref{alg:stoch_relax_ctr} (which relaxes the versioning in \ref{alg:stoch}) using Python constructs. It has three hyper-parameters: (a) $T$ denoting the number of rounds, (b) the exploration rate $p$, and (c) $\epsilon$ controlling the mistake tolerance. For each run, the test data points were randomly permuted, and the first $T$ of them were presented to the algorithm. 

There are two main departures from the scheme in the main text. Firstly, rather than only using feedback gained when $C_t=1,$ the version space is refined whenever $\pred_t = \dk,$ allowing faster learning. Secondly, the versioning is relaxed as already described, to only exclude functions that make too many mistakes, as determined by $\epsilon$.

An important implementation detail is that for very small $\epsilon,$ the version space may get empty before the run concludes. This is particularly relevant for small values of $\epsilon$. As a simple fix, we modify the versioning rule so that if the version space were to become empty at the end of a round, it is not updated (and, indeed, the state of the scheme is retained, see below). 

Since our experiments are CPU compute bounded, we used a machine with two Intel Xeon 2.60 GHz CPUs providing 40 cores. Both the regret-with-varying-time experiments took about $1$ hour compute time, and the operating point experiments took nearly $5$ hours each.

        \begin{algorithm}[t]
        \caption{\textsc{vue-prod-relaxed}}\label{alg:stoch_relax_ctr}
        \begin{algorithmic}[1]
            \State \textbf{Inputs}: $\mathcal{F},$ Exploration rate $p$, Learning rate $\eta$, Tolerance $\epsilon$.
            \State \textbf{Initialise}: $\mathcal{V}_1 \gets \mathcal{F}; \forall t, \mathcal{U}_t \gets \varnothing;  \forall f \in \mathcal{F}, w_1^f \gets 1, o_0^f \gets 0; \mathrm{Ctr}_0 \gets 0$.
            \For{$t \in [1:T]$}
                \State Sample $f_t \sim \pi_t = \frac{w_t^f\indi\{f \in \vs_t\}}{\sum_{f \in \vs_t} w_t^f}.$
                \State Toss $C_t \sim \mathrm{Bern}(p)$.
                \State $\displaystyle \pred_t \gets \begin{cases} \dk & C_t =1 \\ f_t(X_t) & C_t = 0\end{cases}.$
                \If{$\pred_t = \dk$} \Comment{Refine the version space if the exploratory coin is heads}
                        \State $\mathrm{Ctr}_t \gets \mathrm{Ctr}_{t-1} + 1$.
                        \For{$f \in \mathcal{V}_{t}$}
                            \State $o_t^f \gets o_{t-1}^f +  \indi\{f(X_t) \not\in \{\dk, Y_t\}$
                            \If{$o_t^f \le \epsilon \mathrm{Ctr}_{t} + \sqrt{2 \epsilon \mathrm{Ctr}_t}$} \Comment{Retain all $f$s that have error rate $< \epsilon$ w.h.p.}
                                \State $\mathcal{U}_t \gets \mathcal{U}_t \cup \{f\}.$
                            \EndIf
                        \EndFor
                        \State $\mathcal{V}_{t+1} = \mathcal{V}_t \cap \mathcal{U}_t$.
                \Else 
                    \State $\mathcal{V}_{t+1} \gets \mathcal{V}_t$.
                    \State $\forall f \in \mathcal{V}_{t+1}, o_t^f \gets o_{t-1}^f$
                    \State $\mathrm{Ctr}_t \gets \mathrm{Ctr}_{t-1}$.
                \EndIf
                \If{$\vs_{t+1}\neq \varnothing$} \Comment{Penalise Abstentions if the version space is non-empty}
                    \For{$f \in \mathcal{V}_{t+1}$}
                         \State $a_t^f \gets \mathds{1}\{f(X_t) = \dk\}$
                         \State $w_{t+1}^f \gets w_t^f \cdot (1-\eta a_t^f).$
                     \EndFor
                \Else \Comment{$\vs_{t+1} = \varnothing$, and so revert the state}
                    \State $\vs_{t+1} \gets \vs_t.$
                    \State $\mathrm{Ctr}_t \gets \mathrm{Ctr}_{t-1}$.
                    \For{$f \in \vs_{t+1}$} 
                        \State $o_t^f \gets o_{t-1}^f.$
                        \State $w_{t+1}^f \gets w_{t}^f.$
                    \EndFor
                \EndIf
            \EndFor
        \end{algorithmic}
    \end{algorithm}

        \begin{algorithm}[t]
        \caption{\textsc{vue-prod-relaxed-time-adapted}}\label{alg:stoch_relax_adapt_to_T_ctr}
        \begin{algorithmic}[1]
            \State \textbf{Inputs}: $\mathcal{F},$ Tolerance $\epsilon$.
            \State \textbf{Initialise}: $\mathcal{V}_1 \gets \mathcal{F}; \forall t, \mathcal{U}_t \gets \varnothing;  \forall f \in \mathcal{F}, w_1^f \gets 1, o_0^f \gets 0; \mathrm{Ctr}_0 \gets 0$.
            \For{$t \in [1:T]$}
                \State $p_t \gets \min(0.1, 1/\sqrt{t}).$
                \State $\eta_t \gets p_t$.
                \State Sample $f_t \sim \pi_t = \frac{w_t^f\indi\{f \in \vs_t\}}{\sum_{f \in \vs_t} w_t^f}.$
                \State Toss $C_t \sim \mathrm{Bern}(p_t)$.
                \State $\displaystyle \pred_t \gets \begin{cases} \dk & C_t =1 \\ f_t(X_t) & C_t = 0\end{cases}.$
                \If{$\pred_t = \dk$} \Comment{Refine the version space if the exploratory coin is heads}
                        \State $\mathrm{Ctr}_t \gets \mathrm{Ctr}_{t-1} + 1$.
                        \For{$f \in \mathcal{V}_{t}$}
                            \State $o_t^f \gets o_{t-1}^f +  \indi\{f(X_t) \not\in \{\dk, Y_t\}$
                            \If{$o_t^f \le \epsilon \mathrm{Ctr}_{t} + \sqrt{2 \epsilon \mathrm{Ctr}_t}$} \Comment{Retain all $f$s that have error rate $< \epsilon$ w.h.p.}
                                \State $\mathcal{U}_t \gets \mathcal{U}_t \cup \{f\}.$
                            \EndIf
                        \EndFor
                        \State $\mathcal{V}_{t+1} = \mathcal{V}_t \cap \mathcal{U}_t$.
                \Else 
                    \State $\mathcal{V}_{t+1} \gets \mathcal{V}_t$.
                    \State $\forall f \in \mathcal{V}_{t+1}, o_t^f \gets o_{t-1}^f$
                    \State $\mathrm{Ctr}_t \gets \mathrm{Ctr}_{t-1}$.
                \EndIf
                \If{$\vs_{t+1}\neq \varnothing$} \Comment{Penalise Abstentions if the version space is non-empty}
                    \For{$f \in \mathcal{V}_{t+1}$}
                         \State $a_t^f \gets \mathds{1}\{f(X_t) = \dk\}$
                         \State $w_{t+1}^f \gets w_t^f \cdot (1-\eta_t a_t^f).$
                     \EndFor
                \Else \Comment{$\vs_{t+1} = \varnothing$, and so revert the state}
                    \State $\vs_{t+1} \gets \vs_t.$
                    \State $\mathrm{Ctr}_t \gets \mathrm{Ctr}_{t-1}$.
                    \For{$f \in \vs_{t+1}$} 
                        \State $o_t^f \gets o_{t-1}^f.$
                        \State $w_{t+1}^f \gets w_{t}^f.$
                    \EndFor
                \EndIf
            \EndFor
        \end{algorithmic}
    \end{algorithm}

\subsection{Regret Behaviour as Time-horizon in Varied.}

We use the hyperparameter $\epsilon = 0.01.$ For the sake of efficiency, we use the adaptive scheme Algorithm~\ref{alg:stoch_relax_adapt_to_T_ctr} that adapts to the time horizon, that instead varies $p$ with the number of rounds as $p_t = \min(0.1, \frac{1}{\sqrt{t}}), \eta_t = p_t$. This adaptation strategy is a standard way to handle varying horizons, and the observations obtained via this represent (and slightly overestimate) the regrets for when Algorithm \ref{alg:stoch_relax_ctr} is run with $p = \eta = \frac{1}{\sqrt{T}}$. A major advantage is that this significantly increases the efficiency of the procedure, since instead of re-starting the experiment for each time horizon, we can now run for one single time horizon, and obtain representative values of regret at smaller horizons by recording the values at checkpoints corresponding to these. In the plots, we ran for $T = 4000,$ and checkpointed every $250$ rounds.

\subsubsection{Excess Abstention Behaviour}

As noted in the main text, the excess abstention regret for both datasets is negative. This remains consistent with the theory, and likely arises since these datasets are, of course, not the worst case distributions. The excess abstentions regret are plotted below.

\begin{figure}[htb]
    \centering    
    \includegraphics[width = 0.45\textwidth]{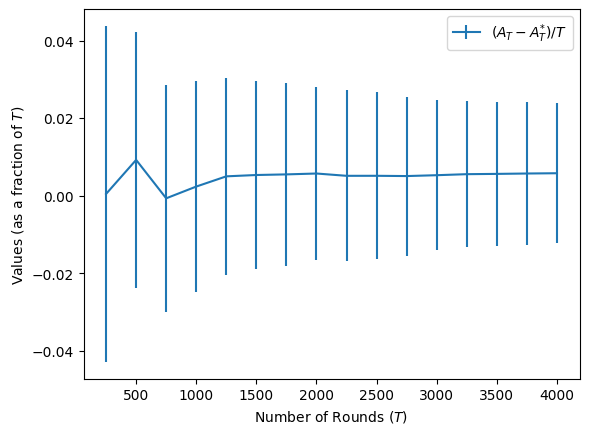}~\includegraphics[width=0.45\textwidth]{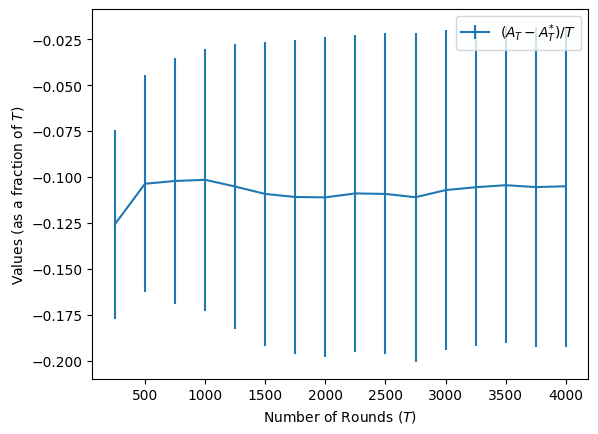}
    
    \caption{ Excess abstention regret, normalised by $T$, in the setting of Figure \ref{fig:vary_T} for CIFAR-10 (left) and GAS (right). The plots are averaged over $100$ runs, and one-standard-deviation error bars are drawn. Notice that the values are negative for GAS, and strongly dominated by the MMEA for CIFAR.}
    \label{fig:abs_vary_T}
\end{figure}

\subsection{Achievable Operating Points of Mistakes and Abstentions}

We use Algorithm \ref{alg:stoch_relax_ctr}, instantiated with $T=500$, and always choosing $\eta = p$. The particular values of $p, \epsilon $ that are scanned are, as listed in the main text, 20 equally spaced values of $p$ in the range $[0.015, 0.285]$, and 10 equally spaced values of $\epsilon$ in the range $[0.001, 0.046]$, giving in total 200 values of $(p, \varepsilon)$ pairs that are scanned over. 

The post-hoc batch operating points are obtained as follows: We first find the largest value of the number of mistakes that are attained by the online learner for some choice of $(p,\varepsilon)$. Call this $M$. The values attained were $M_{\textrm{CIFAR}} = 50$ and $M_{\textrm{GAS}} = 120$. Then, we then instantiated the set $\mathcal{M}_{\textrm{CIFAR}} = \{ 2, 3, \dots, 50\}$, and for $\mathcal{M}_{\textrm{GAS}} = \{2, 7, \dots, 117\}$. The density was chosen lower for GAS for visual pleasantness. 
Finally, for each $m \in \mathcal{M}_*,$ we run the post-hoc optimisation \[ a(m) := \min_{f \in \mathcal{F}} \sum_{t} \indi\{f(X_t) = \dk\} \quad \textrm{s.t. } \quad \sum_{t} \indi\{f(X_t) \not\in \{\dk,Y_t\} \} \le m.\] The resulting points $(a(m), m)$ are plotted as black triangles.

\textbf{Definition of MMEA} As stated in the main text, the mistake matched competitor is defined as follows: suppose that the scheme makes $M$ mistakes and $A$ abstentions over a stream. If the following program is feasible, then we define \begin{align*} A^*(m) &= \min_{f \in \mathcal{F}} \sum \indi\{f(X_t) = \dk\} \textrm{ s.t. } \sum \indi\{f(X_t) \not\in\{\dk, Y_t\} \} \le M.
\end{align*}
If not, then we take $A^*(M)$ to be the abstentions made by the least mistake $f$, which is the competitor in the rest of the section. Then we define \[ \mathrm{MMEA} = A - A^*(M).\]

\subsection{Sensitivity of the scheme to hyperparameters}\label{appx_sensitivity_wrt_eps}

Working in the setting of Figure \ref{fig:vary_T}, we show how the excess mistake and abstention regrets vary at $T = 4000$ (the final point) as $\epsilon$ is varied in Figure \ref{fig:vary_T_with_eps}. As expected, the excess mistakes increase roughly linearly with large $\epsilon$, but the data reflects subtle non-monotonicities in the same. The variation in abstentions is, as expected, essentially opposite to that of the mistakes.

Similarly, in Figure \ref{fig:vary_p_with_eps}, we show the operating points that can be achieved by varying $\epsilon$ for a fixed $p$, and by varying $p$ for a fixed $\epsilon$. We observe first that the variation with $\epsilon$ for a fixed $p$ is relatively regular, with larger $\epsilon$ increasing mistakes but decreasing abstentions at roughly the same rate, up to small variations. On the other, the behaviour with increasing $p$ for a fixed $\epsilon$ is much more subtle, and indicates that a sweet-spot of the coin-based exploration rate exists for each tolerance level.

Together, these plots indicate that the optimal tuning of $\epsilon$ and $p$ together can be subtle, and exploring how one can execute the same in an online way is an interesting open problem.

\begin{figure}
    \centering
    \includegraphics[width = 0.48\textwidth]{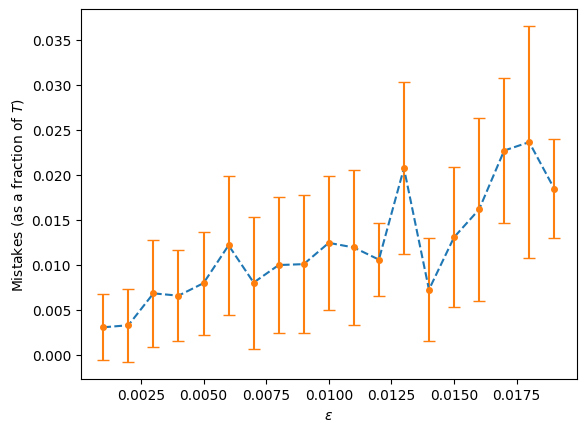}~\includegraphics[width = 0.48\textwidth]{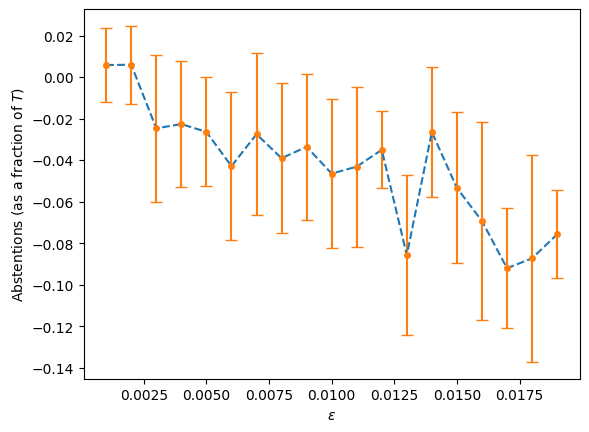}\\
    \includegraphics[width = 0.48\textwidth]{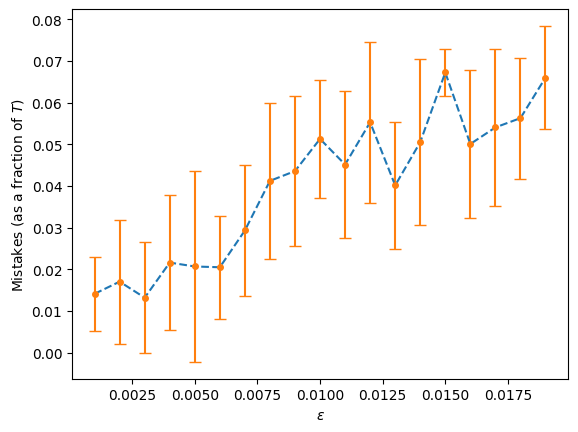}~\includegraphics[width = 0.48\textwidth]{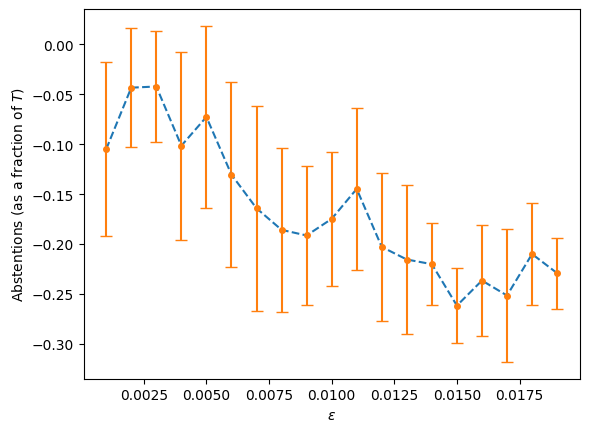}
    \caption{Senstivity with $\epsilon$ of the excess mistakes (left) and excess abstention (right) regrets at $T = 4000$ for CIFAR (top) and GAS (bottom) datasets. Points are averaged over 100 runs, and  one-standard-deviation error bars are included. }
    \label{fig:vary_T_with_eps}
\end{figure}

\begin{figure}
    \centering
    \includegraphics[width = 0.48\textwidth]{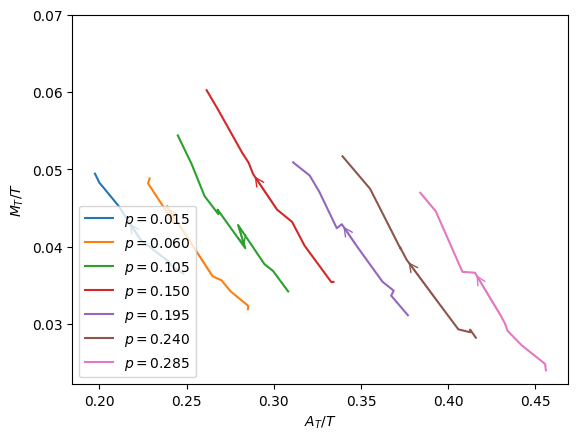}~\includegraphics[width = 0.48\textwidth]{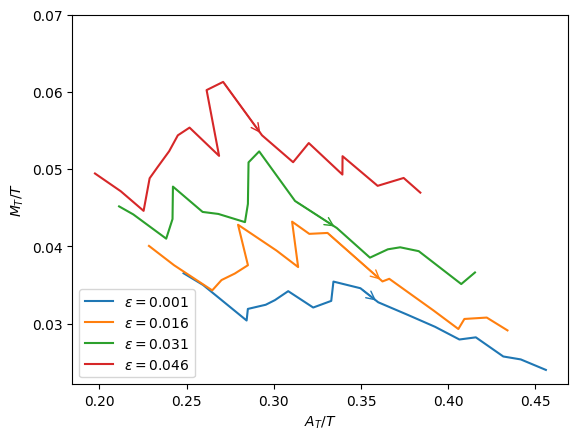}\\
    \includegraphics[width = 0.48\textwidth]{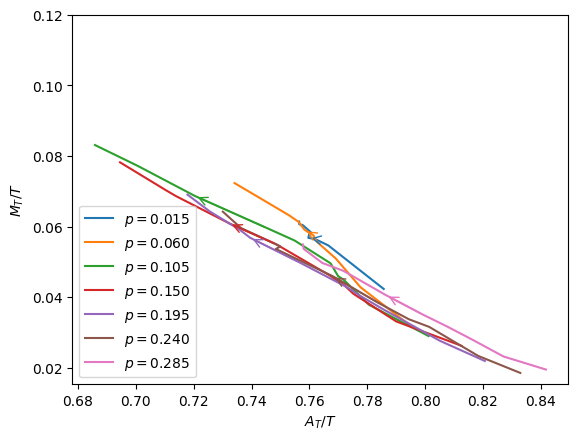}~\includegraphics[width = 0.48\textwidth]{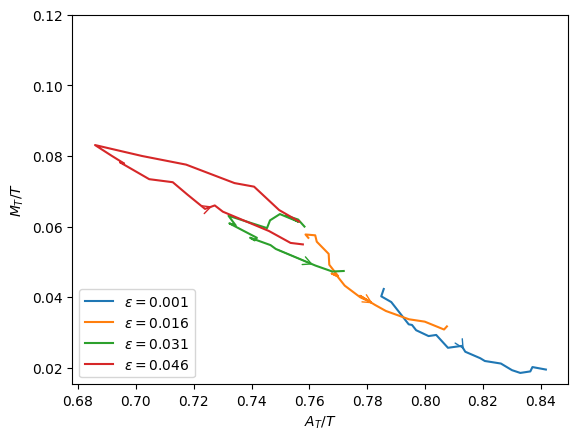}
    \caption{Illustration of how operating points achieved by the scheme vary as $p$ is changed for fixed values of $\epsilon$ (left) and as $\epsilon$ is changed for fixed values of $p$ (right), in the CIFAR (top) and GAS (bottom) datasets. The sets of $\epsilon$s and $p$s marking the traces is reduced with respect to Figure \ref{fig:vary_p} for the sake of legibility. The arrow denotes the direction of increasing the varied parameter.}
    \label{fig:vary_p_with_eps}
\end{figure}

\end{appendix}

\end{document}